%% file: main.tex
\documentclass[sn-apa,iicol]{sn-jnl}

\usepackage{mathrsfs}%
\usepackage[title]{appendix}%
\usepackage{textcomp}%
\usepackage{manyfoot}%
\usepackage{listings}%
\usepackage[caption=false]{subfig}
\input{preamble/common}
\newcommenter{dawei}{Orange}

\begin{document}

\title[Article Title]{Safe and Stable Teleoperation of Quadrotor UAVs under Haptic Shared
Autonomy}

\author*[1]{\fnm{Dawei} \sur{Zhang}}\email{dwzhang@bu.edu}

\author[1]{\fnm{Roberto} \sur{Tron}\email{tron@bu.edu}}

\affil[1]{\orgdiv{Mechanical Engineering}, \orgname{Boston University}, \orgaddress{\street{110 Cummington Mall}, \city{Boston}, \postcode{02215}, \state{MA}, \country{USA}}}

\abstract{We present a novel approach that aims to address both safety and stability of a haptic teleoperation system within a framework of Haptic Shared Autonomy (HSA). 
We use Control Barrier Functions (CBFs) to generate the control input that follows the user's input as closely as possible while guaranteeing safety. In the context of stability of the human-in-the-loop system, we limit the force feedback perceived by the user via a small $\cL_2$-gain, which is achieved by limiting the control and the force feedback via a differential constraint. Specifically, with the property of HSA, we propose two pathways to design the control and the force feedback: Sequential Control Force (SCF) and Joint Control Force (JCF). Both designs can achieve safety and stability but with different responses to the user's commands.
We conducted experimental simulations to evaluate and investigate the properties of the designed methods. We also tested the proposed method on a physical quadrotor UAV and a haptic interface.}


\maketitle

\section{Introduction}\label{sec1}
\input{introduction}

\section{Preliminaries}\label{sec:pre}
\input{preliminaries}

\section{Methods}\label{sec:methods}
\input{methods}

\section{Validation}\label{sec:experiments}
\input{validation}

\section{Conclusions and Future work}
In this paper, we focus on both safety and stability of a haptic teleoperation system. We propose two novel approaches, SCF and JCF, that design the control command $u$ and force feedback $F$ under a Haptic Shared Autonomy (HSA) framework,  which is different from the Haptic Shared Control (HSC) in our previous work \citep{zhang2021stable}.  In both SCF and JCF, we use Control Barrier Function (CBF) constraints to generate a control $u_{\text{cbf}}$ safely; we achieve stability by limiting the force $F$ through a condition that implies a small $\cL_2$-gain approach. The effectiveness of the proposed designs is rigorously evaluated through experimental simulations. In these experiments, we describe the differences between the two designs concerning the responses to the same user commands, so that the system's responses can be designed by considering the properties of each method. Additionally, practical feasibility is demonstrated through testing on a physical quadrotor UAV, showcasing the applicability and viability of the proposed methods in real-world implementation.

The current work is developed on the framework of pre-defined CBFs. In the future, we will further investigate the application of our method on physical UAVs using CBFs that are adaptive to varying scenarios.

\begin{appendices}

\input{appendix}

\end{appendices}

\bibliography{biblio/IEEEfull,biblio/IEEEConfFull,biblio/OtherFull,biblio/references, biblio/haptics}

\end{document}

%% file: preamble/common.tex
\input{preamble/commonNoTikz}
\input{preamble/graphicsTikz}
\input{preamble/robotics}


%% file: preamble/commonNoTikz.tex
\input{preamble/graphics}

\input{preamble/utilities}
\input{preamble/math}
\input{preamble/operators}

\input{preamble/notation}

\input{preamble/units}

\input{preamble/markupAndCommenting}

\input{preamble/pagination}

%% file: preamble/graphics.tex
\makeatletter
\@ifpackageloaded{xcolor}{}{%
\usepackage[table,x11names,dvipsnames,svgnames]{xcolor}%
}
\makeatother

\usepackage{colortbl}

\usepackage{graphicx}
\usepackage{wrapfig}

\input{preamble/graphicsColors}


%% file: preamble/graphicsColors.tex
\definecolorset{RGB}{lyft}{}{Red,194,39,36;Sunset,202,53,33;Orange,205,68,20;Amber,200,117,42;Yellow,242,169,52;Citron,186,188,44;Lime,112,159,33;Green,56,139,31;Mint,45,118,56;Teal,52,133,135;Cyan,60,132,202;Blue,55,94,248;Indigo,64,13,247;Purple,115,42,248;Pink,176,25,145;Rose,176,32,75}

%% file: preamble/utilities.tex
\usepackage{suffix}

\usepackage{environ}

\makeatletter
\newsavebox{\boxifnotempty}
\newcommand{\displayifnotempty}[3]{\sbox\boxifnotempty{#2}\setbox0=\hbox{\usebox{\boxifnotempty}\unskip}%
\ifdim\wd0=0pt
\else
 #1\usebox{\boxifnotempty}#3%
\fi%
}
\newcommand{\ifempty}[2]{\setbox0=\hbox{#1\unskip}%
\ifdim\wd0=0pt%
 #2%
\fi%
}
\newcommand{\ifnotempty}[2]{\setbox0=\hbox{#1\unskip}%
\ifdim\wd0>0pt%
 #2%
\fi%
}
\makeatother

\usepackage{algpseudocode}
\usepackage{algorithm}

\usepackage{scrlfile}

\makeatletter
\newcommand*\newstoreddef[1]{
  \BeforeClosingMainAux{%
    \immediate\write\@auxout{%
      \string\restoredef{#1}{\csname #1\endcsname}%
    }%
  }%
}
\newcommand*{\restoredef}[2]{
  \expandafter\gdef\csname stored@#1\endcsname{#2}%
}
\newcommand*{\storeddef}[1]{
  \@ifundefined{stored@#1}{0}{\csname stored@#1\endcsname}%
}
\makeatother

\usepackage{pageslts}
\NewEnviron{tee}{\BODY\typeout{Marker Tee [start] ^^J \BODY ^^JMaker Tee [end]}}


%% file: preamble/operators.tex
\newcommand{\real}[1]{\mathbb{R}^{#1}{}}

\newcommand{\bmat}[1]{\begin{bmatrix}#1\end{bmatrix}}

\newcommand{\smallbmat}[1]{\left[\begin{smallmatrix}#1\end{smallmatrix}\right]}

\newcommand{\transpose}{^\mathrm{T}}

\newcommand{\defeq}{\doteq}

\DeclarePairedDelimiter{\norm}{\lVert}{\rVert}

\newcommand{\de}{\mathrm{d}}

\DeclareMathOperator*{\argmin}{\arg\!\min}

\newcommand{\subjectto}{\textrm{subject to }}


%% file: preamble/notation.tex
\providecommand{\cC}{\mathcal{C}}

\providecommand{\cH}{\mathcal{H}}

\providecommand{\cL}{\mathcal{L}}


%% file: preamble/units.tex
\usepackage{units}


%% file: preamble/markupAndCommenting.tex
\newcommand{\newcolorlabel}[2]{%
  \expandafter\newcommand\csname #1\endcsname[1]{%
    \colorbox{#2}{\color{white}\textsf{\textbf{##1}}}}%
}

\newcommand{\newcommenter}[2]{%
  \expandafter\newcommand\csname #1\endcsname[1]{%
    \fcolorbox{#2}{#2}{\color{white}\textsf{\textbf{#1}}}
    {\color{#2}##1}}%
  \expandafter\newcommand\csname at#1\endcsname{%
    \fcolorbox{#2}{#2}{\color{white}\textsf{\textbf{@#1}}}
    {\color{#2}}}%
  \expandafter\newcommand\csname #1hl\endcsname[2]{%
    \colorbox{#2}{\color{white}\textsf{\textbf{#1}}}\sethlcolor{Azure2}\hl{##2}~%
    \expandafter\ifx\csname commentarrow\endcsname\relax$\leftarrow$\else \commentarrow[#2]\fi~%
    {\color{#2}##1}}%
  \expandafter\newcommand\csname #1st\endcsname[2]{%
    \colorbox{#2}{\color{white}\textsf{\textbf{#1}}}\sout{##2}~%
    \expandafter\ifx\csname commentarrow\endcsname\relax$\leftarrow$\else \commentarrow[#2]\fi~%
    {\color{#2}##1}}%
}
\newcommenter{TODO}{DodgerBlue1}
\newcommenter{rtron}{Green3}

\usepackage{comment}

\usepackage{pdfcomment}

\usepackage{soul}

\usepackage[normalem]{ulem}

\usepackage{csquotes}


%% file: preamble/pagination.tex
\usepackage{microtype}

\usepackage{array}
\usepackage{multirow}
\usepackage{booktabs}
\usepackage{makecell} 

\ifcsname labelindent\endcsname

\fi
\usepackage[inline]{enumitem}

\newenvironment{lenumerate}[2][]
{\begin{enumerate}[label=(#2\arabic*),leftmargin=0.2in,itemindent=0.15in,#1]}
{\end{enumerate}}

\setlist*[enumerate,1]{label={\itshape\arabic*)}}

\makeatletter
\newcommand{\paragraphswithstop}{%
\let\copyparagraph\paragraph%
\renewcommand\paragraph[1]{\copyparagraph{##1.}}%
}
\makeatother

\usepackage[framemethod=tikz]{mdframed}


%% file: preamble/graphicsTikz.tex
\usepackage{tikz}
\usetikzlibrary{calc}
\usetikzlibrary{matrix}
\usetikzlibrary{chains}
\usetikzlibrary{shapes.geometric}
\usetikzlibrary{arrows.meta}
\usetikzlibrary{decorations.pathreplacing}
\usetikzlibrary{backgrounds}

\tikzset{
  dim above/.style={to path={\pgfextra{
        \pgfinterruptpath
        \draw[>=latex,|->|] let
        \p1=($(\tikztostart)!1.5em!90:(\tikztotarget)$),
        \p2=($(\tikztotarget)!1.5em!-90:(\tikztostart)$)
        in(\p1) -- (\p2) node[pos=.5,sloped,above]{#1};
        \endpgfinterruptpath
      }
    }
  },
  dim double above/.style={to path={\pgfextra{
        \pgfinterruptpath
        \draw[>=latex,|->|] let
        \p1=($(\tikztostart)!3em!90:(\tikztotarget)$),
        \p2=($(\tikztotarget)!3em!-90:(\tikztostart)$)
        in(\p1) -- (\p2) node[pos=.5,sloped,above]{#1};
        \endpgfinterruptpath
      }
    }
  },
  dim below/.style={to path={\pgfextra{
        \pgfinterruptpath
        \draw[>=latex,|->|] let 
        \p1=($(\tikztostart)!-1em!-90:(\tikztotarget)$),
        \p2=($(\tikztotarget)!-1em!90:(\tikztostart)$)
        in (\p1) -- (\p2) node[pos=.5,sloped,below]{#1};
        \endpgfinterruptpath
      }
    }
  },
}

\tikzset{
    right angle quadrant/.code={
        \pgfmathsetmacro\quadranta{{1,1,-1,-1}[#1-1]}     
        \pgfmathsetmacro\quadrantb{{1,-1,-1,1}[#1-1]}},
    right angle quadrant=1, 
    right angle length/.code={\def\rightanglelength{#1}},   
    right angle length=2ex, 
    right angle symbol/.style n args={3}{
        insert path={
            let \p0 = ($(#1)!(#3)!(#2)$) in     
                let \p1 = ($(\p0)!\quadranta*\rightanglelength!(#3)$), 
                \p2 = ($(\p0)!\quadrantb*\rightanglelength!(#2)$) in 
                let \p3 = ($(\p1)+(\p2)-(\p0)$) in  
            (\p1) -- (\p3) -- (\p2)
        }
    }
}

\newcommand{\pgfextractangle}[3]{%
    \pgfmathanglebetweenpoints{\pgfpointanchor{#2}{center}}
                              {\pgfpointanchor{#3}{center}}
    \global\let#1\pgfmathresult  
}

\usetikzlibrary{shapes.arrows}
\newcommand{\commentarrow}[1][Azure4]{\tikz[baseline=-3pt]{\node[shape border uses incircle, fill=#1,rotate=180,single arrow, inner sep=1pt, minimum size=6pt, single arrow head extend=2pt]{};}}

%% file: preamble/robotics.tex
\tikzset{ax/.style={-latex,line width=2pt}}

\tikzset{camera/.style={fill=Sienna1,fill opacity=0.5},%
image plane/.style={draw=RoyalBlue3,line width=2pt}}


%% file: introduction.tex

Teleoperation serves as a pivotal tool, enabling human operators to engage in activities within dangerous or challenging environments where autonomous control is impractical. This capability is particularly crucial in tasks that demand human intuition and adaptability, such as search-and-rescue operations in intricate and cluttered environments. In the context of UAVs, teleoperation presents a unique set of challenges, primarily stemming from the inherent difficulty in safely and precisely controlling the UAV. This challenge is exacerbated by the restricted field of view, leading to diminished situational awareness for the operator, thereby demanding innovative solutions to enhance control precision and operational safety in UAV teleoperation scenarios.~\citep{mccarley2005human, Brandt2010}. 

Haptic feedback emerges as a promising technique to remedy the aforementioned challenges. A notable paradigm leveraging haptic feedback is Haptic Shared Control (HSC), wherein haptic signals serve as informative cues regarding the robot's behavior and the surrounding environment through force feedback. Taking this approach often reduces dangerous collisions during teleoperation, and has been shown to increase operator situational awareness \citep{Lam2009, Brandt2010, zhang2020haptic}. 

In the framework of HSC, the human operator can either adhere to haptic suggestions or override them as needed. In teleoperation of UAVs, it has been repeatedly shown that, although HSC can improve a user's ability to fly the UAV, instances of human operators unintentionally crashing the vehicles remain a recurrent challenge, hindering overall operational safety. To address this limitation, an alternative paradigm, Haptic Shared Autonomy (HSA), has been recently proposed. HSA guarantees safety by applying autonomous methods as a supervisory controller and utilizing haptic feedback to inform the user about the discrepancy between the original command and the command that is modified by the controller. This approach contributes to enhancing human-robot agreement and user satisfaction \citep{zhang2021haptic}. 
Various methods to develop such supervisory controllers have been explored in the literature \citep{schwarting2017parallel, Xu2018, broad2018learning}. In this paper, we leverage the framework of control barrier functions (CBFs) \citep{Ames2014,Ames2019,zhang2021haptic} to the design of the autonomous controller.

Beyond safety considerations, stability stands as a pivotal element in a teleoperation system, particularly in the presence of haptic feedback. As the controller generates force feedback in response to user commands, and users, in turn, adapt their inputs based on this force feedback, the closed-loop response of the overall system becomes susceptible to instabilities. These instabilities often manifest as uncontrolled and unexpected oscillations in the haptic interface. Our objective is to mitigate the aforementioned stability challenge through the application of a finite-$\cL_2$-gain approach. It is essential to note that this work does not cover other stability traits such as robustness to delays. While some of the previous work has considered this problem, existing approaches are confined to the HSC paradigm, and face limitations in terms of generalizability to diverse settings with varying system dynamics or employing different haptic feedback strategies (refer to Sec.~\ref{sec:related_work} for additional discussion). 

In this paper, our focus lies on addressing both safety and stability aspects in the design of an autonomous controller. This controller acts as an intermediary between the user and the robot, delivering force-based feedback to the user and managing the control delegated to the robot.

\subsection{Related Work}\label{sec:related_work}
HSC is a teleoperation paradigm that aims to help the human user safely navigate the robot through haptic cues while maintaining the user's control authority \citep{zhang2021haptic}. Numerous works within this paradigm concentrate on the algorithmic design of haptic feedback. For example, virtual fixtures (also known as virtual mechanisms \citep{joly1995imposing}) have been widely employed to generate haptic feedback, directing the human operator when commanding the robot to protected areas defined by these fixtures \citep{abbott2007haptic, bowyer2013active,rosenberg1993virtual,payandeh2002application,li2005constrained,turro2001haptically}. Regarding the teleoperation of UAVs, a parametric risk field (PRF) method was proposed by \cite{Lam2009} to generate force feedback based on the risk of collision. The work of \cite{Brandt2010} sets the magnitude of the haptic feedback to be proportional to the time that it would take the UAV to collide with an object in its environment. Our prior work introduced an approach that uses Control Barrier Functions (CBFs) \citep{Ames2014, Ames2019} to guide the human toward the input command that is closest to their current command and deemed to be safe \citep{zhang2020haptic}. These methods under HSC successfully improve the safety and transparency of a teleoperation system; however, a limitation of such approaches is that safety is not guaranteed when the human user overrides the safe guidance.

Instead of providing haptic cues, several studies ensure safety through a Shared Autonomy (SA) paradigm, wherein autonomous controllers can modify the human operator's control command in instances of disagreement between the human and the robot. For instance, \cite{Xu2018} employed CBFs as a supervisory controller that modifies the operator’s control input and guarantees the safe teleoperation of UAVs . Similar SA paradigms include outer-loop stabilization \citep{broad2018learning} and parallel autonomy. For example, Schwarting applied Nonlinear Model Predictive Control (NMPC) to guarantee the safety of human-controlled automated vehicles \citep{schwarting2017parallel}. Taking advantage of both HSC and SA, our prior work proposed a Haptic Shared Autonomy (HSA) paradigm that uses haptic feedback to reflect the inner state of the system when the disagreement between the user and the robot happens \citep{zhang2021haptic}. Although safety has been widely investigated in the literature above, these works lack the consideration of stability, which is another critical aspect of a haptic teleoperation system. 
To our knowledge, only a handful of works consider the stability of the human-robot-environment system, and predominantly within the HSC paradigm.
The stability analysis of haptic teleoperation systems is typically facilitated by properties of the specific system under consideration - an approach that may not generalize when extended such ideas to novel systems. For example, \cite{abbott2003analysis, abbott2006stable} used the eigenvalues of a discrete state-space model to analyze the stability of a forbidden-region virtual fixture (FRVF) system. Input-to-state (ISS) stability with Lyapunov Functions was also applied to prove the stability in the haptic teleoperation of UAVs \citep{rifai2011haptic, Omari2013}. These approaches allow for establishing strong stability guarantees but require a new ad-hoc analysis (design of the Lyapunov function, and proof of convergence) for every new combination of system dynamics and haptic feedback methods, which limits their applications, especially in HSA.

Since passivity provides a sufficient condition for stability, making the system passive provides a convenient approach to maintaining the stability of a teleoperation system \citep{Niemeyer2008}. This principle finds widespread application in bilateral teleoperation systems \citep{ryu2005time,hannaford2002time,kosuge1995tele,selvaggio2019passive,adams1999stable}. Lee et al. introduced the Passive-Set-Position-Modulation (PSPM) method, which modulates the input to enforce system passivity. They applied PSPM to the haptic teleoperation of multiple UAVs, ensuring system passivity over the Internet despite varying delays and packet loss  \citep{lee2010passive, lee2011haptic}. Similarly, wave variables \citep{niemeyer2004,Niemeyer2008,lam2007collision} and Port-Hamiltonian methods \citep{mersha2013bilateral,stramigioli2010novel} have been proposed to address issues arising from system delays. However, these approaches mostly concentrate on instabilities associated with delays and often overlook instability induced by human responses within the closed loop.  Furthermore, many of these methods make specific assumptions about the modeled environment and user behavior, constraining their adaptability to settings with diverse system dynamics and haptic feedback. Our recent work proposed a finite-$\cL_2$-gain approach that guarantees the stability in human-in-the-loop haptic teleoperation using a less restrictive differential constraint as compared with traditional passivity methods \citep{zhang2021stable}. This work was built upon a HSC paradigm without considering the properties of the system when shared autonomy \citep{zhang2021haptic} is integrated. Therefore, in this paper, we investigate the stability of a human-robot-environment haptic teleoperation system for quadrotor UAVs under a HSA framework, wherein safety is concurrently guaranteed.

\input{system_architecture}

\subsection{System modeling}
In this study, we consider a haptic teleoperation system in HSA whose architecture is shown in Fig.~\ref{fig:architecture}. In this system, the robot is navigated by the human operator to interact with the remote environment.
The haptic device can display force-based haptic feedback and capture the user's desired command from the controlled joystick. The desired command is passed to the robot system as a reference control signal $u_{\textrm{ref}}$. In turn, the autonomous controller (e.g., defined via $\cL_2$-gain and CBFs in this paper) provides a force $F$ to the user, which creates a feedback loop, as shown in Fig.~\ref{fig:feedback_connect}.

\input{feedback_connection}
From a modeling perspective, we lump together the user and the physical haptic device as a single \emph{Human system}; this is because, from the point view of the supervisory controller, the output and input of the Human system are force feedback $F$ and the command $u_{\text{ref}}$, respectively, and it is not possible to separate the effects of each subsystem on their serial composition; this holds also for the stability analysis that we describe in this paper. 
For similar reasons, we lump together the robot and its response to the environment. Nonetheless, there are aspects of teleoperation that cannot be captured with mathematical models in practice, such as the intentions of the user and disturbances due to imperfect knowledge of the environment. In Fig.~\ref{fig:feedback_connect}, these are represented as \emph{User intention} and \emph{Disturbance} that enter between the human and robot connections, and have the potential to cause unstable behaviors of the overall closed-loop system. Then, the force feedback that goes into the human system is represented by $\tilde{F}$, and the reference control to the robot system becomes $\tilde{u}_{\text{ref}}$.

\subsection {Contributions}

To concurrently attain safety and stability, we establish two primary objectives in the design of both the force $F$ and the control $u$:
\begin{lenumerate}{G}
    \item\label{it:limit control} Generating a control $u$ that follows the user's input ``closely" (i.e., matches the user's intention) while guaranteeing safety; this is achieved through CBFs.
    \item\label{it:limit force} Limiting the force $F$ perceived by the user to achieve stability of the human-in-the-loop system; this is achieved by bounding the $\cL_2$-gain of the autonomous system via a differential constraint similar to~\citep{zhang2021stable}.
\end{lenumerate}

Our previous work only achieved goal~\ref{it:limit force} under the HSC paradigm, where $u_{\textrm{ref}}$ would be directly passed to the robot. 
This paper focuses on the stability under HSA, where safety is guaranteed, to achieve both Goal~\ref{it:limit control} and Goal~\ref{it:limit force}; as summarized in Table \ref{table_summary}, there is no previous work (including from the authors) which considers Goal~\ref{it:limit control}, Goal~\ref{it:limit force}, and HSA at the same time.

\begin{table}[t]
\centering
\caption {Organization of the previous work with respect to the present paper. Note that the two columns correspond to not considering (left) or considering (right) Goal~\ref{it:limit force}.}
\begin{tabular}{ |p{5mm}|p{2.8cm}|p{2.8cm}| }
\hline
  & Stability not guaranteed & Stability guaranteed \\ 
 \hline
 HSC & \cite{Lam2009, Brandt2010,zhang2020haptic,abbott2007haptic,payandeh2002application,li2005constrained}& \cite{lee2011haptic, abbott2006stable, rifai2011haptic, zhang2021stable} \\ 
 \hline
 HSA & \cite{zhang2021haptic} & This paper  \\  
 \hline 
\end{tabular}
\label{table_summary}
\end{table}
Regarding HSA, the human's control command can be altered by the autonomous controller to guarantee safety, therefore, we have two pathways to limit $u$ and $F$:
\begin{description}
    \item[Sequential Control Force (SCF):] Compute $u$ first, and then compute $F$ based on $u$;
    \item[Joint Control Force (JCF):] Compute $u$ and $F$ in the same optimization problem, so that $u$ can be chosen in a way that also limits $F$.
\end{description} 

The key contributions of this paper are listed as follows:
\begin{itemize}
    
    \item We propose two designs (SCF and JCF) that consider the properties of HSA; both designs can achieve the safety goal~\ref{it:limit control} and the stability goal~\ref{it:limit force}.
    \item We describe the differences in the responses to the user commands so that a particular design strategy can be picked according to the specific application and user preference. 
    \item In contrast to conventional approaches, our method offers flexibility in application to various nonlinear systems. This adaptability is realized by computing the necessary Lie derivatives for the dynamics and Control Barrier Functions.
    \item For each designed method, we provide the quadratic programming formulation or closed-form solutions to make the optimization problem solvable in real-time. This capability is crucial for robotic teleoperation scenarios, ensuring the practicality and efficiency of our proposed methods.
\end{itemize}

%% file: system_architecture.tex
\tikzset{system/.style={very thick, draw=DodgerBlue3}}

\begin{figure}
    \centering
    \includegraphics{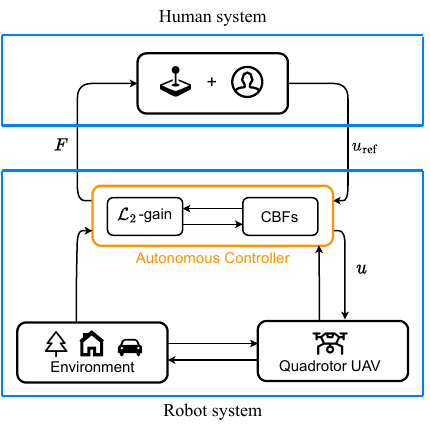}
    \caption{\small An architecture of the proposed haptic teleoperation system.}
    \label{fig:architecture}
\vspace{-10pt}
\end{figure}

%% file: feedback_connection.tex
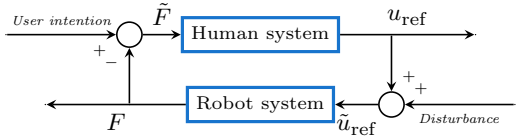
\begin{figure}[ht]
\centering
\begin{tikzpicture}[
point/.style={circle,inner sep=0pt,minimum size=0.5pt,fill=black},
point_to/.style={->,>=stealth,semithick},
squarednode/.style={rectangle, draw=black, fill=black!0, semithick, minimum size=2mm},
circlenode/.style={circle, draw=black, fill=black!0, semithick, minimum size=2mm},]
\matrix[row sep=-2mm,column sep=5mm] {
\node[point](p1){};&&\node[circlenode] (c1) [label= -175:\tiny $+$,label= -95:\tiny $-$]{}; & \node[squarednode,system] (Human) {\footnotesize Human system}; & \node[point](p2){};& \node[point](p3){};& & \\
& \node[point](p4){};&\node[point](p5){}; & \node[squarednode,system] (Quad) {\footnotesize Robot system}; & \node[circlenode] (c2)[label=5:\tiny $+$,label=85:\tiny $+$]{};& & \node[point](p6){}; & \\
};
\draw[point_to] (p1) to node[auto] {\tiny \textit{User intention}} (c1);
\draw[point_to] (c1) to node[auto] {\small $\tilde{F}$} (Human);
\draw[point_to] (Human) to node[auto] {\small $u_{\textrm{ref}}$} (p3);
\draw[point_to] (p6) to node[auto] {\tiny \textit{Disturbance}} (c2);
\draw[point_to] (c2) to node[auto] {\small $\tilde{u}_{\text{ref}}$} (Quad);
\draw[point_to] (Quad) to node[auto] {\small $F$} (p4);
\draw[point_to] (p5) to (c1);
\draw[point_to] (p2) to (c2);
\end{tikzpicture}

\caption{\small Feedback connection of the haptic teleoperation architecture. Note that the Human system and the Robot system are the same with the two systems highlighted by the blue boxes in Fig.~\ref{fig:architecture}.}
\label{fig:feedback_connect}
\vspace{-5pt}
\end{figure}

%% file: preliminaries.tex
In this section, we comprehensively review and introduce key concepts pivotal for the subsequent sections of this paper.
\subsection{Control Barrier Functions}
\subsubsection{State Space Model}
Consider a dynamical system represented by the state space model
\begin{equation}\label{state model}
  \begin{aligned}
    \dot{x}&=f(x)+g(x)u
\end{aligned}
\end{equation}
where $x\in\real{n}$ is the state of the system, $u\in\real{d}$ represents the vector of control inputs , and $f: \real{n} \rightarrow \real{n}$, $g:\real{n}\rightarrow\real{n}\times\real{d}$ are locally Lipschitz.

In this paper, we specialize \eqref{state model} to the dynamics of a quadrotor. We assume that the quadrotor flies at relatively low speeds without aggressive maneuvers (which are exceedingly uncommon in a teleoperation setting) so that the roll and pitch angles of the quadrotor will remain small, and the yaw angles can be controlled independently to maintain the first-person view. Under such conditions, the dynamics of the UAV can be approximated by a double integrator, where the control input $u\in\real{d}$ corresponds to the acceleration command of the UAV. Additionally, the boundaries of the input device and the saturation limits of the actuators can be neglected.
Let $x=\smallbmat{x_p\\x_v}\in\real{2d}$ be the state of the quadrotor, where $x_p\in\real{d}$ represents its position and $x_v = \dot{x}_p\in\real{d}$ its velocity. The dynamics of the system can be written as:
\begin{equation}\label{eq:double integrator}
\begin{array}{l}
    \bmat{\dot{x}_p\\\dot{x}_v} = \bmat{0&I\\0&0}\bmat{x_p\\x_v} + \bmat{0\\I}u,
\end{array}
\end{equation}
where $I$ is an identity matrix of appropriate dimensions.

\subsubsection{Lie derivatives}
We denote the Lie derivative of a continuously differentiable function $h(x)$ along a vector field $f(x)$ as $L_{f} h(x)\defeq\frac{\partial h(x(t))}{\partial x(t)} f(x)$.
We denote with $L_{f}^{b} h(x)$ a Lie derivative of order $b$. The function $h$ has relative degree two with respect to the dynamics \eqref{state model} if $L_gh=0$, and $L_gL_fh$ is a non-singular matrix. In this case, we have $\ddot{h}=L_{f}^{2} h(x)+L_{g} L_{f} h(x)u$.
\subsubsection{Safety Set}
We define the safe set $\cH$ as the zero superlevel set of a continuously differentiable function $h(x)$:
\begin{equation}\label{eqn:safety_set}
\cH: = \left\{ x \in\real{n} : h (x)\geq 0 \right\}.
\end{equation}
For any initial condition $x_0 \in \cH$, there exists a maximum interval of existence $I(x_0) = [0, \tau_{max})$ such that $x(t)$ is the unique solution to system \eqref{state model} on $I(x_0)$.
\begin{definition}\label{def:forward invariance}
The set $\cH$ is forward invariant for system \eqref{state model} if for every $x_0 \in \cH$, $x(t) \in \cH$ for $x(0)=x_0$ and $\forall t \in I(x_0)$.
\end{definition}

\subsubsection{CBFs for Second Order Systems}

The goal of CBFs is to produce a control field $u$ that makes a \emph{safe set} $\cH\subset\real{n}$ forward invariant \citep{Ames2019}. In our problem settings, we define $h(x)$ as the Euclidean distance between the robot and the obstacles and utilize CBFs to avoid collisions. Let $h(x)$ be a twice differentiable function representing $\cH$, i.e. $h(x)>0$ on the interior of $\cH$, $h(x)=0$ on its boundary, and $h(x)<0$ otherwise. Assuming that $h(x)$ has relative degree two, we can use a second-order exponential control barrier function \citep{Nguyen2016} to impose constraints on $u$ that ensure safety (i.e., forward invariance of $\cH$):
\begin{equation}\label{eq:HOCBF}
 L_{f}^{2} h(x)+L_{g} L_{f} h(x) u
+K\bmat{h(x)& L_{f}h(x)}^T \geq 0,
\end{equation}
where $K\in\real{1\times 2}$ is a set of coefficients representing a Hurwitz polynomial. Given the system dynamics in~\eqref{eq:double integrator}, we can rewrite the constraint in \eqref{eq:HOCBF} as:
\begin{equation}\label{eq:CBF_constraint}
\tag{CBF}
   x\transpose_v\partial^2_{x_p} h x_v+ (\partial_{x_p} h)\transpose u+ k_1h + k_2(\partial_{x_p} h)\transpose x_v\geq 0.
\end{equation}

\subsection{$\cL_2$-gain}\label{subsec:l2_gain}
\begin{definition}\label{def:L two}
A map $\cC:u(t)\to y(t)$ between two signals has $\cL_2$-gain $\gamma\geq 0$ if there exists a constant $\beta\in\real{}$ such that
\begin{equation}\label{eq:L2 gain}
\norm{y}_2\leq \gamma \norm{u}_2+\beta.
\end{equation}
\end{definition}
Note that the map $\cC$ could be static (i.e., a simple function) or, more commonly, realized through a dynamical system. The importance of this concept is given by the \emph{small gain theorem} (reproduced below with a specialization to our settings that are described in Fig.~\ref{fig:feedback_connect}).
\begin{theorem}[Theorem 5.6, page 218,~\citep{khalil2002nonlinear}]Assume that both Human system and Robot system are finite-gain $\cL_2$ stable with $\cL_2$ gains of $\gamma_1$ and $\gamma_2$: $\norm{u_{\text{ref}}}_2 \leq \gamma_1\norm{\tilde{F}}_2 + \beta_1$ and $\norm{F}_2 \leq \gamma_2\norm{\tilde{u}_{\textrm{ref}}}_2 + \beta_2$. If $\gamma_1 \gamma_2<1$, then the feedback connection is finite-gain $\cL_2$ stable from the inputs $(\textit{User intention}, \textit{Disturbance})$ to the outputs $(\tilde{F}, \tilde{u}_{\textrm{ref}})$.
\end{theorem} 

In this study, we assume that  the human's reactions to the force feedback correspond to a map with a finite $\cL_2$ gain.

%% file: methods.tex
In this section, we first introduce an approach that limits the force feedback $F$ using finite-$\cL_2$-gain in an energy-based formulation via a linear inequality involving the time derivatives of $u$ and $F$. Then, we discuss the application of the CBF framework to our problem setting and introduce two novel design approaches for $u$ and $F$: the Sequential Control-Force (SCF) design and the Joint Control-Force (JCF) design. Finally, we briefly discuss the feasibility consideration of the proposed optimization problem.

\subsection{Limiting $F$ via Finite-$\cL_2$-gain}

In this section, we discuss an energy-based formulation that ensures the finite-gain $\cL_2$ stability from the users' desired velocity $x_{vd}$ to the force feedback $F$ that is perceived by the user. We also introduce the design of the storage function and the energy tank that are used to derive the differential constraint for the $\cL_2$ stability.
\subsubsection{Energy design}
We design the storage function to represent the mechanical energy stored in the system, therefore, we identify the function $V(x)$ as :
\begin{equation}\label{eq:storage function}
    V(x)=\frac{k_v}{2}\norm{Bx}^2=\frac{k_v}{2}x_v\transpose x_v,
\end{equation}
where $k_v$ is a positive constant parameter.
\subsubsection{Energy tank}
To make the constraint less restrictive, we introduce an energy tank $E$ that is used to store energy when the reference force naturally satisfies the stability constraint and releases energy when the reference force violates that same constraint. Formally, we view $E$ as another state in the system, with dynamics
\begin{equation}
  \label{eq:tank dynamics}
  \dot{E}=\varepsilon,
\end{equation}
where $\epsilon \in \real{}$ is an additional control we design for the energy flow of the tank.
\subsubsection{$\cL_2$ constraint on the force}
As introduced in \citep{zhang2021stable}, we use a differential constraint that implies a finite $\cL_2$ gain between $F$ and $x_{vd}$:
\begin{subequations}
    \begin{align}
        \frac{k}{2}\norm{F}^2+\varepsilon= \frac{1}{2k} \norm{x_{vd}}^2 - \dot{V}(u,x),&\label{eq:finite gain tank}\\
    \varepsilon\geq - \frac{E}{\Delta_t}.& \label{eq:Edot-nonnegative}
    \end{align}
\end{subequations}

\begin{proposition}\label{prop_1}
If conditions~\eqref{eq:finite gain tank} and~\eqref{eq:Edot-nonnegative} are satisfied for every $t\geq 0$, then the signals $F$ and $x_{vd}$ satisfy the $\cL_2$-gain definition~\ref{def:L two} with $\gamma = \frac{1}{k^2}$ and $\beta= \frac{2}{k} V(0) + E(0)$, where $V(0)$ is the initial energy of the system, and $E(0)$ is the initial energy of the energy tank.
\end{proposition}
\begin{proof}
Taking the integral of both sides of constraint~\eqref{eq:finite gain tank} and using~\eqref{eq:Edot-nonnegative}, we can obtain
\begin{multline}
    \int_0^ t \norm{F(\tau)}^2 \de \tau + \int_0^ t \dot{E} \de \tau = \int_0^ t \norm{F(\tau)}^2 \de \tau \\
    + E(t)-E(0)
    \leq  \frac{1}{k^2}\int_0^t\norm{x_{vd}(\tau)}^2\de \tau+\frac{2}{k} V(0).
\end{multline}

Additionally, the constraint~\eqref{eq:Edot-nonnegative} implies that $E(t) \geq0$, which implies
\begin{multline}
    \int_0^ t \norm{F(\tau)}^2 \de \tau
    \leq \frac{1}{k^2}\int_0^t\norm{x_{vd}(\tau)}^2\de t \\+ \frac{2}{k} V(0) + E(0),
\end{multline}
from which the claim follows by comparing with Definition~\ref{def:L two}.
\end{proof}

Conditions~\eqref{eq:finite gain tank} and~\eqref{eq:Edot-nonnegative} can be equivalently written as:
\begin{equation}
\norm{F}^2 \leq \frac{2}{k}( \frac{E}{\Delta_t} +\frac{1}{2k} \norm{x_{vd}}^2 - \dot{V}(u,x)).
\end{equation}
Note that this is a quadratic and convex constraint, but is also singular with respect to $u$ (since there are no quadratic terms in $u$).
By substituting the storage function in \eqref{eq:storage function}, we can rewrite the constraint as:
\begin{equation}\label{eq:l2_force}
\tag{$\cL_2$-output}
  \norm{F}^2 \leq \frac{1}{k^2}( \frac{2kE}{\Delta_t} +x_{vd}\transpose x_{vd}) - \frac{2k_v}{k} x_v\transpose u.
\end{equation}

To ensure feasible solutions for the force $F$, we need to constrain the RHS of \eqref{eq:l2_force} to be non-negative, which leads to following constraint on $u$:
\begin{equation}\label{eq:kF feasibility constraint}
\tag{$\cL_2$-feasibility}
    \frac{2kE}{\Delta_t}+x_{vd}\transpose x_{vd} -2kk_vx_v\transpose u\geq 0,
\end{equation}
from which we can notice that a greater value of $k_v$ gives stricter limits on $u$. Note that it is always feasible when $u=0$ or when $u$ has an opposite direction with $x_v$ (i.e. $x_v\transpose u \geq 0$). 


\subsection{Reference controller}
We define a reference proportional controller as
\begin{equation}\label{eq:uref}
    u_{\textrm{ref}}= \frac{1}{\Delta_t}(x_{vd}-x_v),
\end{equation}
where $x_{vd}$ is the desired velocity set by the user, $x_v$ is the current velocity of the robot,  and $\Delta_t$ is a time constant representing how long $u_{\textrm{ref}}$ will be applied to the robot (i.e., $x_v$ will become $x_{vd}$ after $\Delta_t$, i.e., in a single step).

\subsection{Sequential Control-Force Design}
In this methodology, we systematically calculate the control signal $u$ followed by the sequential computation of the force $F$. During the computation of $u$, we account for both the \eqref{eq:CBF_constraint} constraint and the \eqref{eq:kF feasibility constraint} constraint. Conversely, in the computation of $F$, our considerations are solely directed towards the \eqref{eq:l2_force} constraint.

\subsubsection{Control design}
We use a CBF-QP formulation~\citep{Ames2019} to find the input $u$ that is closest to the reference $u_{\textrm{ref}}$ while satisfying the safety and feasibility constraints. Again, we note that while the force $F$ does not appear explicitly in this optimization problem, its feasibility is captured in the constraint of \eqref{eq:kF feasibility constraint}.
\begin{equation}\label{eq:CBF-Feasibility}
\begin{aligned}
u_{\text{scf}} = &\underset{u \in \mathbb{R}^{d}}{\argmin}{\frac{1}{2}\norm{u -u_{\text{ref}}}^{2}}\\
\subjectto & \eqref{eq:CBF_constraint}, \\
&\eqref{eq:kF feasibility constraint}.
\end{aligned}
\end{equation}

\subsubsection{Force design}
We design a reference force  $F_{\textrm{ref}}$ that depends on the discrepancy between $u_{\text{ref}}$ and the returned $u_{\text{scf}}$ as done in \citep{zhang2020haptic}:
\begin{equation}
    F_{\text{ref}}= u_{\textrm{scf}} -u_{\text{ref}}.
\end{equation}
Then we formulate a new force synthesis problem that returns a force that is as close as possible to the reference force $F_{\text{ref}}$ while also satisfying the $\cL_2$-gain constraint of \eqref{eq:l2_force}.
\begin{equation}\label{eq:optimal F with tank}
    \begin{aligned}
    F_{\text{scf}} = &\min_{F \in \mathbb{R}^{d}} \frac{1}{2}\norm{F-F_{\text{ref}}}^2\\
      \subjectto& \eqref{eq:l2_force}.
    \end{aligned}
\end{equation}

\subsection{Joint Control-Force Design}

In this section, we design an effective approach that solves the control input $u$ and the force feedback $F$ jointly through the following synthesis problem:
\begin{equation}\label{eq:JCF}
\begin{aligned}
u_{\text{jcf}}, F_{\text{jcf}}=&\argmin_{u,F \in \mathbb{R}^{d}} w_{\text{cbf}}\norm{u -u_{\text{ref}}}^{2}\\ &\hspace{5ex}+w_{\cL_2}\norm{F-(u-u_{\text{ref}})}^2\\
    \subjectto &\eqref{eq:CBF_constraint},\\
      &\eqref{eq:l2_force},
\end{aligned}
\end{equation}
where $w_{\text{cbf}}$ and $w_{\cL_2}$ are constant parameters to adjust the weights of $u$ and $F$ in the optimization. 

This optimization problem can be solved as a Second Order Cone Program (SOCP). However, SOCP solvers are typically much slower than QP solvers, thus limiting their application for control in real time. Moreover, we mostly focus on environments in $\real{2}$, where at most one or two obstacles are active at a time. For these reasons, we use a method to solve all cases that will happen under this formulation. We consider six cases that are listed as follows, and then pick the one that is both feasible and gives the minimum cost. For most of the cases, we provide a closed-form solution (albeit possibly requiring finding roots of a polynomial).
Informally, we consider a constraint active if removing it would change the solution for the specific problem instance. In general cases, this means that the constraint evaluates to zero at the solution. Technically, it means that it has a non-zero Lagrange multiplier \citep{Bertsekas99}.
\begin{lenumerate}{C}
    \item \label{C1} \textit{No constraint is active.} This becomes an unconstrained problem. By inspection, we get that $u=u_{\textrm{ref}}$ and $F=0$ give a cost of zero, which gives the optimal solution.
    
    \item \label{C2} \textit{The quadratic constraint is not active.} Since there is no active constraint on $F$, we can always select $F =u-u_{\textrm{ref}}$, so that the second term in the cost is always zero.
    
    The linear constraint can be obtained by replacing the inequality in~\eqref{eq:CBF_constraint} with equality. If there is only one obstacle, we can obtain a closed-form solution through the method of Lagrange multipliers in which we find the minima of the objective function in \eqref{eq:JCF} under the single equality constraint. We provide the details of solving this problem in Appendix \ref{appendix_1}. If there are multiple obstacles, the problem becomes a Quadratic Program, which can still be solved efficiently online with a CBF-QP formulation \citep{Ames2019}.

    \item \label{C3} \textit{Only the quadratic constraint is active.} The quadratic constraint can be obtained by replacing the inequality in \eqref{eq:l2_force} with equality.
    
    Similar to the previous one with only one linear constraint, we can use the method of Lagrange multipliers which, however, results in finding the roots of a third-order polynomial. When multiple solutions are feasible, we pick the one that gives the minimum cost. In Appendix~\ref{appendix_2}, we provide the details, including the calculation of $u$ and $F$ from the roots of the polynomial.

    \item \label{C4} \textit{The quadratic constraint and only one linear constraint are active.}
    At a high level, we first perform an algebraic transformation to remove the linear constraint, so that we obtain a problem of the same form as the previous case. In particular, we use the decomposition $u=U_\bot u_{\bot}+\partial_{x_p} h u_{\parallel}$, where $U_\bot\in \real{d-1}$ is an orthonormal basis for the orthogonal complement of $\partial_{x_p} h$ in $\real{d}$ (i.e., $U_\bot\transpose \partial_{x_p}h=0$ and $U_\bot\transpose U_\bot=I_{d-1}$). $U_\bot$ is computed using a Singular Value Decomposition. $u_{\parallel}\in\real{}$ and $U_\bot\in \real{d-1}$ represent the new coordinates for $u$.

Substituting the decomposition in \eqref{eq:CBF_constraint} with equality, we obtain
\begin{multline}
   x\transpose_v\partial^2_{x_p}h x_v + (\partial_{x_p} h)\transpose \partial_{x_p}h u_{\parallel}+ k_1h \\+ k_2(\partial_{x_p} h)\transpose x_v = 0,
\end{multline}
from which we can solve for $u_{\parallel}$ using the method of Lagrange multipliers, which is similar to the calculations in Appendix \ref{appendix_1}.

Then the optimization problem becomes
\begin{equation}\label{eq:QCQP reduced}
\begin{aligned}
\min_{u'\in\real{}, F\in\real{d} } &  w_{\text{cbf}}\norm{U_{\bot} u' + \partial_{x_p}h u_\parallel-u_{\text{ref}}}^2\\+ &w_{\cL_2}\norm{F-(U_{\bot} u' + \partial_{x_p}h u_\parallel-u_{\text{ref}})}^2\\
\subjectto&  \norm{F}^2 = \frac{1}{k^2}( \frac{2kE}{\Delta_t} +x_{vd}\transpose x_{vd}\\ &-2kk_vx_v\transpose (\partial_{x_p}h u_\parallel +U_\bot u')).
    \end{aligned}
\end{equation}

Given the properties of $U_\bot$, for any vector $v$, we can write $\norm{v}^2=\norm{U_\bot\transpose v}^2+\frac{1}{\norm{\partial_{x_p}h}^2}\norm{\partial_{x_p} h\transpose v}^2$. Applying this to \eqref{eq:QCQP reduced} and removing constant terms, we get

\begin{equation}
\begin{aligned}
\min_{u'\in\real{}, F\in\real{d}}& w_{\text{cbf}}\norm{ u' + U_{\bot}\transpose(\partial_{x_p}h u_\parallel-u_{\textrm{ref}})}^2\\+w_{\cL_2}\norm{U_{\bot}&\transpose F -\bigl( u' + U_{\bot}(\partial_{x_p}h u_\parallel-u_{\text{ref}})\bigr)}^2\\ +w_{\cL_2}&\frac{1}{\norm{\partial_{x_p}h}^2}\norm{\partial_{x_p} h\transpose F\\ &- \partial_{x_p}h\transpose \bigl(\partial_{x_p}h u_\parallel-u_{\text{ref}}) }^2\\
\subjectto & \norm{U_{\bot}\transpose F}^2+\frac{1}{\norm{\partial_{x_p}h}^2}\norm{\partial_{x_p} h\transpose F}^2 \\= &\frac{1}{k^2}( \frac{2kE}{\Delta_t} +x_{vd}\transpose x_{vd} \\ &-2kk_vx_v\transpose (\partial_{x_p}h u_\parallel +U_\bot u')).
    \end{aligned}
\end{equation}

We introduce the variables $u'_{\text{ref}}=U_\bot\transpose u_{\text{ref}} \in \real{d-1}$, $u''_{\text{ref}}=\frac{\partial_{x_p}h \transpose}{\norm{\partial{x_p}h}} u_{\textrm{ref}}\in \real{}$, $F'=U_\bot\transpose F\in\real{d-1}$, and $F''=\frac{\partial_{x_p}h \transpose}{\norm{\partial{x_p}h}} F\in\real{}$, such that $u_{\text{ref}}=U_\bot u'_{\text{ref}}+\frac{\partial_{x_p}h }{\norm{\partial{x_p}h}} u''_{\text{ref}}$, $F=U_\bot F'+\frac{\partial_{x_p}h}{\norm{\partial{x_p}h}}F''$. Then the optimization problem can be formulated as

\begin{equation}
\begin{aligned}
\min_{u',F',F''} &w_{\text{cbf}}\norm{ u' - u'_{\text{ref}}}^2\\&+w_{\cL_2}\norm{F'-( u' - u'_{\text{ref}})}^2 \\&+w_{\cL_2} \norm{F''-( u_\parallel-u''_{\text{ref}}) }^2\\
\subjectto&  \norm{F'}^2+\norm{F''}^2\\ = &\frac{1}{k^2}( \frac{2kE}{\Delta_t} +x_{vd}\transpose x_{vd} \\&-2kk_vx_v\transpose (\partial_{x_p}h u_\parallel +U_\bot u')),
    \end{aligned}
\end{equation}
from which we can obtain the values of $u'$, $u''$, $F'$, and $F''$ through the methods described for the previous cases (detailed calculations are provided in Appendix \ref{appendix_3}).

    \item \label{C5} \textit{The quadratic constraint and two linear constraints are active, or the coordinate dimension $d$ is greater than two, and there are d active linear constraints.}
    In this case, the $d$ linear constraints determine a single feasible point. By assumption, at this point the quadratic constraint is satisfied. Hence, if this case were to happen, it would be covered by case \ref{C2}.

    \item\label{C6} \textit{The coordinate dimension $d$ is greater than two; the quadratic and fewer than $d$ linear constraints are active.}
    We would need to consider all possible cases of the quadratic constraint being active with up to $d-1$ linear constraints. Each of these cases can be handled similarly to case \ref{C4} (where $u_\perp$ is defined to be in the nullspace of all active constraints). Since we consider only $d=2$, this case is out of the scope of this work. 
\end{lenumerate}

\subsection{Infeasibility Consideration}
For both SCF and JCF designs, the $\cL_2$ constraint \eqref{eq:l2_force} might conflict with the CBF constraint \eqref{eq:CBF_constraint}, which is caused by the fact that, to avoid collisions, the CBF controller might have to inject energy into the system (i.e. change $\dot{V}(u,x)$ faster than what is allowed by the input $x_{vd}$). When this case happens, we only keep the CBF safety constraint and set $F=0$. Assuming that the infeasible condition lasts for a finite period, the $\cL_2$ condition described in \ref{subsec:l2_gain} could still be satisfied, but with a larger $\beta$; i.e., following Proposition \ref{prop_1}, this corresponds to a solution where we initialized the tank with non-zero initial energy $E(0)$. One can keep track of this value, and report it to the user as an indication of how loose the $\cL_2$ bound might be.

%% file: validation.tex
In this section, we first introduce a simulation and an experiment to demonstrate the qualitative effectiveness of our $\cL_2$-gain constraint in avoiding instabilities in the system. We then provide quantitative results for simulations where a quadrotor is navigated in an environment to highlight the difference between the responses given by the SCF and JCF designs. Additionally, we provide a comparison between the SCF design and a baseline method using the traditional passivity constraint. Finally, we talk about our implementation on the real quadrotor and haptic device to demonstrate the feasibility of the proposed method.

\subsection{Examples of instability and $\cL_2$-gain constraint}
\subsubsection{Simulated human response}

In order to show an example of the instability that may happen in our system, we approximate the human's responses to the force as a spring-damper model, which can be described as:
\begin{equation}\label{eq:human model}
    \Ddot{x}_{vd} + p\Dot{x}_{vd} + q(x_{vd}-x_{v0}) = F,
\end{equation}
where $x_{vd}$ is the commanded velocity, $x_{v0}$ is a human's set velocity, and $F$ is the force feedback that is perceived by the human. \input{instability} Here, we use a stiffness parameter $q$ and a damping parameter $p$ to approximate how aggressive the responses to the force will be (respectively, how intensely and how quickly the human pushes back the force). Since the actual human responses are in general time-varying and nonlinear, equation. \eqref{eq:human model} will provide, at best, a very coarse approximation of reality; nonetheless, it is sufficient to show how a typical case of instability can arise.

We first couple the model \eqref{eq:human model} with the SCF version of our haptic system (the results for the JCF formulation are qualitatively similar), but without considering the $\cL_2$-gain constraint of \eqref{eq:l2_force} while generating the force $F$. 
Fig.~\ref{instability_a} shows the resulting force feedback $F$ and the human's commanded velocity $x_{vd}$: as apparent from the plots, the system enters into bounded but sustained oscillation. This type of response is generally undesired by the user. Fig.~\ref{instability_b} shows similar plots for the case where the small $\cL_2$-gain constraint is enforced: in this case, the amplitude of the oscillations decays exponentially. The time constant of the decay can be changed by tuning the $\cL_2$ gain $k$.

\subsubsection{Real human response}

\begin{figure}[ht]
\centering
\label{trajectory}
\includegraphics[width=0.9\columnwidth,trim={0.2cm 0.2cm 0cm 0cm},clip]{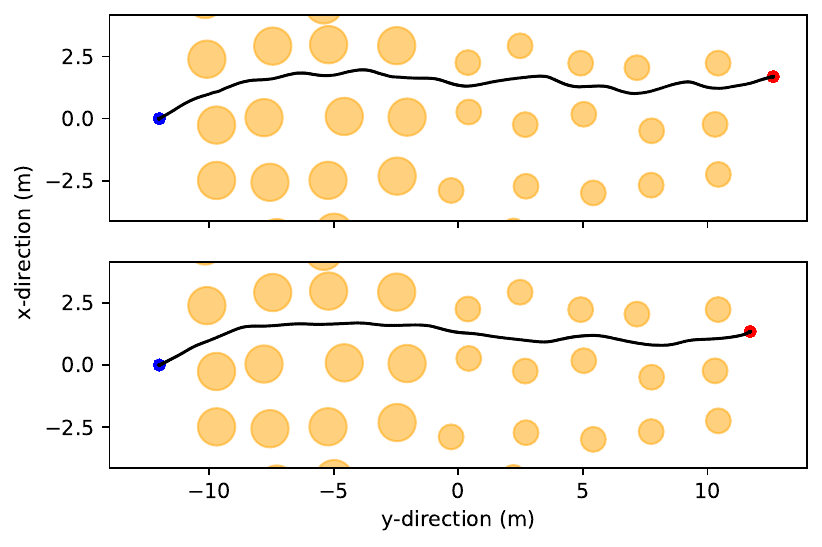}
\caption{Trajectories of the quadrotor under the condition without $\cL_2$-gain method (above) and the condition with $\cL_2$-gain method (below). The blue dot represents the start position while the red dot the end position. Yellow circles denote the obstacles.}
\label{Fig:trajectory}
\end{figure}

\input{instability_human}
We record real user responses when a simulated quadrotor is flown in an environment with multiple obstacles, as shown in Fig.~\ref{Fig:trajectory}. Since the human's responses will be affected by the force feedback and the same user may also have slightly different responses, it is impossible to provide a low-level quantitative comparison between different autonomous controllers. Instead, we offer a qualitative evaluation of our proposed method by collecting the data from two separate experiments, which however have the same initial conditions and where the same user roughly exhibits the same behavior. Fig.~\ref{Fig:trajectory} shows the trajectories under the conditions with and without the $\cL_2$-gain constraint. 
Fig.~\ref{Fig:instability_human} shows the force feedback and the human's desired velocity (which also represents how the human responds to the force feedback); the trajectory appears to be smoother when the $\cL_2$-gain constraint is active (Fig.~\ref{instability_c_human} and Fig.~\ref{instability_d_human}) than when it is not (Fig.~\ref{instability_a_human} and Fig.~\ref{instability_b_human}).  
Looking more closely at the plots of the $x$ and $y$ components of the force and velocity  as a function of time, we see unexpected oscillations, when the small $\cL_2$-gain method is not applied (Fig.~\ref{instability_a_human} and Fig.~\ref{instability_b_human}); these oscillations almost disappear when the small $\cL_2$-gain constraint is enforced (Fig.~\ref{instability_c_human} and Fig.~\ref{instability_d_human}).  Please refer also to the supplementary video for additional details.

These qualitative results show that the constraint \eqref{eq:l2_force} offers a practical and effective way to avoid instabilities in the system. 

\subsection{One-dimensional experiment}\label{subsec:1-d}

\subsubsection{Experiment} We implement each one of the proposed approaches in a simulated environment where a UAV is navigated in a space without obstacles, and then toward a wall-shaped obstacle located \unit[$6$]{$m$} away from the starting position, as shown in Fig.~\ref{fig:simulation}. To better investigate the properties of the proposed methods, we preset the user's input as a one-dimensional trapezoidal signal in the y direction shown in Fig.~\ref{fig:simulation}. During the experiment, we record the states of the UAV, the generated force feedback $F$, the generated control $u$ and the reference control $u_{\text{ref}}$, and the energy of the system.

In this experiment, we use the distance from the UAV to the wall as the barrier function, which has the form:
\begin{equation}
    h(x)= -x_p + 6,
\end{equation}
where $x_p \in \real{}$ is the position of the UAV in this one-dimensional setup. The dynamics of the UAV is the same as the double integrator as described by \ref{eq:double integrator}.
\begin{figure}[t]
  \centering
  \includegraphics[width=0.9\columnwidth]{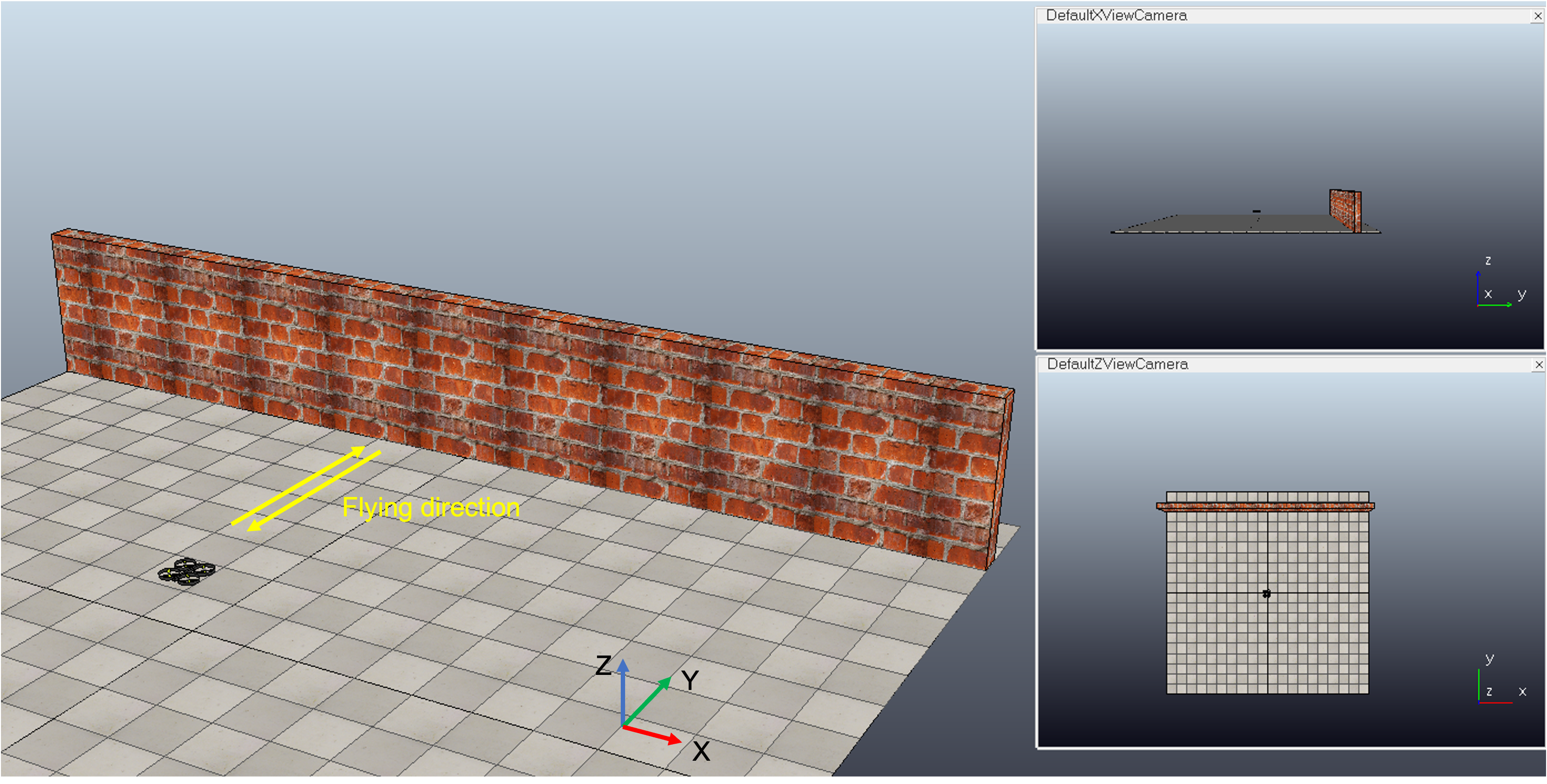}
  \caption{\small A quadrotor UAV is navigated to approach a wall.}
  \label{fig:simulation}
\end{figure}

\subsubsection{Results and discussion}
\begin{figure}[b]
  \centering
  \includegraphics[width=0.9\columnwidth,trim={0cm 0cm 0cm 0cm,clip}]{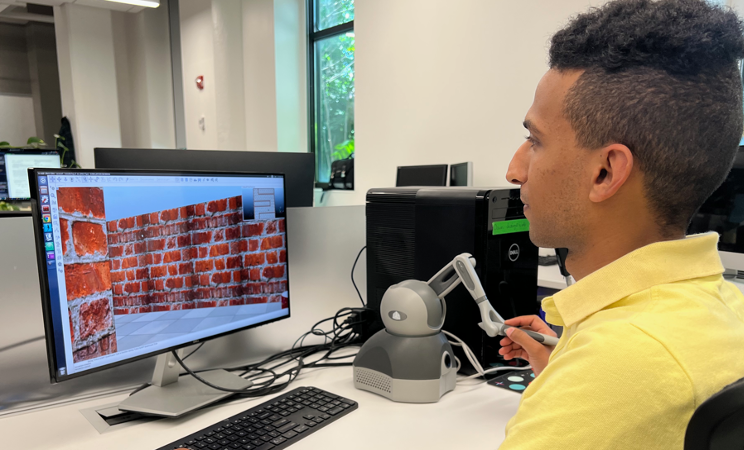}
  \caption{\small Experimental setup of the two-dimensional simulation.}
  \label{Fig:setup}
\end{figure}
\input{results_SCF_no_obstacle}
\input{results_JCF_no_obstacle}

Fig. \ref{Fig:SCF_free} and Fig.~\ref{Fig:JCF_free} are results for SCF and JCF when there are no obstacles. As we can see from \ref{subfig:SCF_b} and \ref{subfig:SCF_f}, when applying the SCF design, there is no force being generated with $k_v = 1$ while there is an ``inertia-like'' resistance force feedback with $k_v = 5$ when the UAV does not reach the desired velocity (e.g. from \unit[$4$]{$s$} to \unit[$8$]{$s$}, as marked by the red box A in Fig.~\ref{subfig:SCF_e}).

When applying the JCF design, resisting force feedback is generated with both $k_v = 1$ and $k_v = 5$, as shown in Fig.~\ref{subfig:JCF_b} and Fig.~\ref{subfig:JCF_f}. We can also see that a greater value of $k_v$ results in greater modifications of the control input, and larger force feedback for a relatively longer time (e.g. from \unit[$3$]{$s$ to \unit[$14$]{$s$}} in Fig.~\ref{subfig:JCF_f}, as highlighted by the red box D).

\input{results_SCF_obstacle}
\input{results_JCF_obstacle}

The results for the two proposed designs with a wall-shaped
obstacle are shown in Fig.~\ref{Fig:SCF_obs} and Fig.~\ref{Fig:JCF_F_obs}. If the UAV is far away from the obstacle, the results for both designs are very similar to the no-obstacle condition. However, both designs (either with $k_v = 1$ or $k_v = 5$) render repulsive force feedback as the UAV gets close to the obstacle (e.g. from \unit[$10$]{$s$ to \unit[$13$]{$s$}} in Fig.~\ref{Fig:SCF_obs_b} and  Fig.~~\ref{Fig:SCF_obs_f}, from \unit[$17$]{$s$ to \unit[$20$]{$s$}} in Fig.~\ref{Fig:JCF_obs_f}, which are highlighted in red boxes F, G, and J). Also, the UAV stops before colliding with the obstacle, as noted by the red box E in Fig.~\ref{Fig:SCF_obs_a} and F in Fig.~\ref{Fig:JCF_obs_a}, indicating that the rendered control $u$ from both SCF and JCF designs can guarantee safety.

\input{results_diff_Emax}

Fig.~\ref{Fig:Force_energy} shows the results of force feedback $F$ with respect to different maximum allowable energy in the designed energy tank. We find that different values of the maximum energy $E_{max}$ can result in different force feedback responses. For both designs, the peak magnitude of the force increases as the value of $E_{max}$ goes up, which aligns well with our previous findings in \citep{zhang2021stable}. Therefore, we can consider $E_{max}$ as a parameter to determine how aggressive the force will be in a very short time period when the UAV approaches the obstacle. Accordingly, $E_{max}$ can be designed and adjusted according to the user's preference; this is reminiscent of the impulsive force approach (initially proposed by \cite{salcudean1997emulation}), which can improve the perceived stiffness of the constraint.

Given the result of the one-dimensional experiment, we can summarize that
\begin{itemize}
    \item both designs can guarantee the safety and stability of the system;
    \item JCF introduces more modifications of the control input $u$ and provides an ``inertia-like" resistance force regardless of the value of $k_v$;
    \item for both SCF and JCF, a greater value of $E_{max}$ results in a sharper and higher peak of the force feedback.
\end{itemize}

\subsubsection{A comparison with the traditional passivity constraint}\label{sec:passivity}
Existing methods do not apply to our setup (see Table~\ref{table_summary}), and cannot be easily extended (since their design is tightly integrated with the assumption of HSC operations). Nevertheless, we implemented a baseline using the traditional passivity constraint \citep{khalil2002nonlinear}, which is one of the main principles used by previous work. Specifically, we replaced our $\cL_2$ stability constraint in the SCF design with the output passivity constraint:
\begin{equation}
    x_{vd}\transpose F \leq k_vx_v\transpose u + k\norm{F}^2.
\end{equation}
It is important to note that we use the same storage function as described in \eqref{eq:storage function}.
Fig.~\ref{Fig:passivity_j} indicates that the passivity constraint tends to be too conservative. For example, the force feedback after \unit[$15$]{s} is zero, while some response from the system would still be expected since the desired velocity $x_{vd}$ is directed toward the obstacle. In contrast, our $\cL_2$ constraint is less conservative, and thus provides more flexibility to design the force feedback, as shown in Fig.~\ref{Fig:SCF_obs_b} and Fig.~\ref{Fig:SCF_obs_f}.

\subsection{Two-dimensional experiment}

\subsubsection{Experimental setup}\label{simulation_setup} In this experiment, instead of using a preset input, a human user navigates the drone using a 3D Systems Touch Haptic Device, as shown in Fig.~\ref{Fig:setup}. To highlight the differences between the SCF and JCF approaches, we compare the responses of the two by replaying the same user input. The quadrotor UAV had a radius of \unit[$0.25$]{m} and is equipped with a forward-facing camera and a bottom-facing camera. The view from the cameras is displayed on a 24-inch computer monitor. The communication between the haptic interface and CoppeliaSim is performed via the Robot Operating System (ROS) middleware. The height of the UAV above the floor is fixed and the user only has control of the robot's horizontal position and yaw angle. The user controls the robot from a first-person perspective so that moving the stylus forward (i.e. away from the user) would result in a motion in the direction of the UAV's front-facing camera, moving the stylus to the left would result in the UAV banking left, and so on. The yaw angle of the UAV is controlled using the two buttons on the stylus of the haptic device, with one button commanding a counterclockwise rotation and the other commanding a clockwise rotation.
The stylus command is mapped to the UAV's commanded velocity $x_{vd}$ through a constant of \unit[$2$]{$\frac{m/s}{cm}$} (i.e. \unit[$1$]{$cm$} displacement corresponds to a desired velocity of \unit[$2$]{$m/s$}). The rate of yaw rotation is \unit[$\frac{\pi}{4}$]{rad/s} when a button is pressed. The supplementary video contains additional details on how the two-dimensional experiment is conducted.

In this experiment, we approximate the rectangle-shaped obstacles as super-ellipses \citep{barr1981}. Specifically, the control barrier function for each obstacle is of the form:
\begin{equation}
    h_i(x)= \left(\frac{x_{p1} - x_{o,i}}{a}\right)^{2a/r} + \left(\frac{x_{p2} - y_{o,i}}{b}\right)^{2b/r} -1,
\end{equation}
where $\smallbmat{x_{p1}\\x_{p2}}$ is the position of the UAV, $\smallbmat{x_{o,i}\\y_{o,i}}$ is the center position of the $i$-th obstacle in the two-dimensional setting. $2a$ and $2b$ represent the length and width of the approximated rectangle with a rounded corner radius $r$, respectively. In the implementation, we use $a = 4.5$, $b = 1.5$, and $r = 0.5$ for all three obstacles.

\subsubsection{Results and discussion}

\input{results_2D}

Fig.~\ref{Fig:2D_simulation} shows the  results of the two-dimensional experiment. In Fig.~\ref{Fig:2D_simulation}, the yellow rectangles represent the obstacles, the gray circles reflect the shape of the UAV, and the red and blue arrows represent the user's control input and the force feedback, respectively. From Fig.~\ref{Fig:SCF_Kf_1} and Fig.~\ref{Fig:SCF_Kf_2}, we can find that there is no obvious difference when changing the value of $k_v$ under the SCF design. However, when comparing Fig.~\ref{Fig:JCF_Kf_1} and Fig.~\ref{Fig:JCF_Kf_2} with Fig.~\ref{Fig:SCF_Kf_1} (as indicated by red boxes K, L and M), we can observe that when applying JCF the force feedback is, in general, not perpendicular to the obstacle and a component of the force feedback is directed against the direction of the velocity. 
Comparing Fig.~\ref{Fig:JCF_Kf_1} and Fig.~\ref{Fig:JCF_Kf_2}, together with the fact that we are using the same user input, we can conclude that a greater value of $k_v$ introduces more significant changes to the trajectories of the UAV (e.g. the trajectories that are highlighted by dashed red circles N and O in the figures), therefore, resulting in different responses of the force feedback. 

Given the result of this two-dimensional experiment, we can conclude that
\begin{itemize}
    \item JCF generates obvious resistant force feedback as compared with SCF;
    \item  SCF follows the CBF-generated safe reference trajectory very closely, while JCF introduces more modifications to the trajectory; a greater value of $k_v$ results in larger modifications.
\end{itemize}
These results are consistent with the findings of the one-dimensional experiment.

\subsection{Hardware validation}\label{sec:hardware_exp}

To validate the viability of our proposed methodology, we execute the SCF technique on an actual quadrotor UAV. Our experiment consists of two trials, mirroring the one-dimensional study outlined in Sec.~\ref{subsec:1-d}. For the sake of simplicity, we affix the target to the center of the video frame in both trials, employing PID control exclusively in the x-axis (sway) and z-axis (heave). Consequently, force feedback and control signals are exclusively exerted in the forward and backward directions.

In the initial trial, the drone is directed toward the target using a predefined rate control signal, gradually increasing from $0$ to \unit[$20$]{$cm/s$}. In the subsequent trial, a human operator pilots the drone using a haptic interface, with the stylus's position on the interface being directly correlated to the rate control signal.

\subsubsection{Hardware setup}
 In this physical implementation, we solve the optimization problem through a desktop computer (8-core, 16GB RAM, Ubuntu 16.04 system), which is the same as our simulation. A DJI Tello drone is employed in the experiment, as it can take velocity control commands, which is necessary in our control framework. The communication between the drone and the desktop is based on Tello's local Wi-Fi. The visual feedback to the user is displayed on a 24-inch monitor. For the perception, we use the AprilTag library \citep{AprilTag} to detect the tag and use the translation of the tag in the image frame as an estimation of the distance to the wall.
 The detailed experimental setup can be seen in Fig.~\ref{fig:hardware_setup}.
 \begin{figure}[ht]
  \centering
  \includegraphics[width=0.9\columnwidth,trim={0cm 0cm 0cm 0cm,clip}]{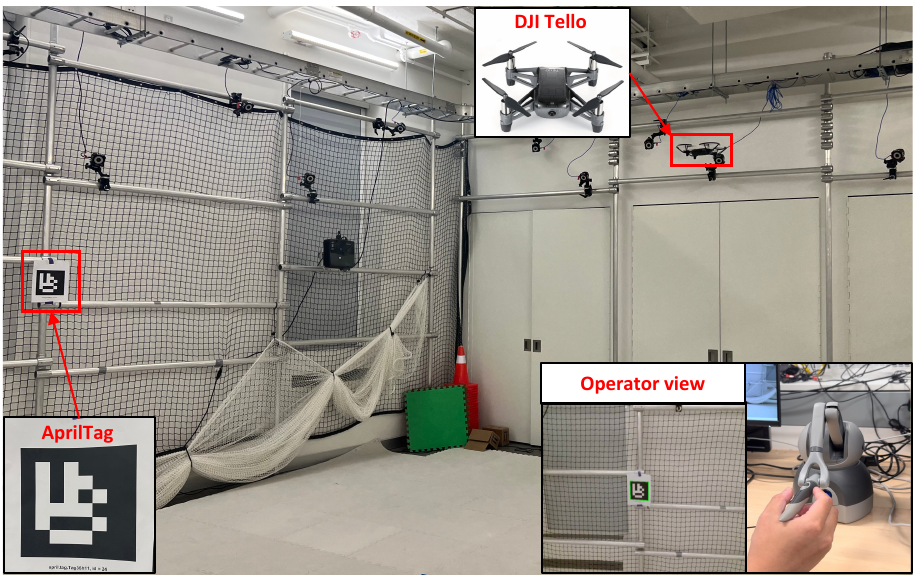}
  \caption{\small Hardware experimental setup.}
 \vspace{-5pt}
  \label{fig:hardware_setup}
\end{figure}
 In the trial of the human user control, a 3D System Touch is employed as the haptic interface, which is also the same setup as described in \ref{simulation_setup}.
 \subsubsection{Results and discussion}
 \input{hardware_rate}
 \input{hardware_user}
 For both trials, we record the states of the UAV, the generated force feedback $F$, the generated control $u$ and the reference control $u_{\text{ref}}$, and the system energy $E$.
 Fig.~\ref{Fig:hardware_rate} shows the results for the trial of the predefined rate control and Fig.~\ref{Fig:hardware_user} shows the results obtained from the user's direct control commands. We can observe that for both trials, the results have the same properties as summarized in our one-dimensional simulations. For instance, the force reaches a peak, then goes to a constant as the drone approaches the wall, which is highlighted by the red box P in Fig.~\ref{Fig:hardware_rate}. Additionally, as the user pilots the drone to the wall, the velocity of the drone gradually decreases to zero, ensuring safety (as indicated by the red box Q in Fig.~\ref{Fig:hardware_user}). During the experiment, we found that the delays from the communications and the robot's actions didn't affect the performance of our approach as long as the parameters in the method were properly adjusted. For example, the state estimation of the drone can be improved by tuning the refresh rate of the control loop. These observations affirm the practicality of our system assumptions, leading to the conclusion that the proposed approaches are indeed feasible for real-world implementation.

%% file: instability.tex
\begin{figure}[t]
\centering
\subfloat[Without $\cL_2$-gain constraint]{{\label{instability_a}}
\includegraphics[width=0.45\textwidth,trim={0cm 0cm 0cm 0cm},clip]{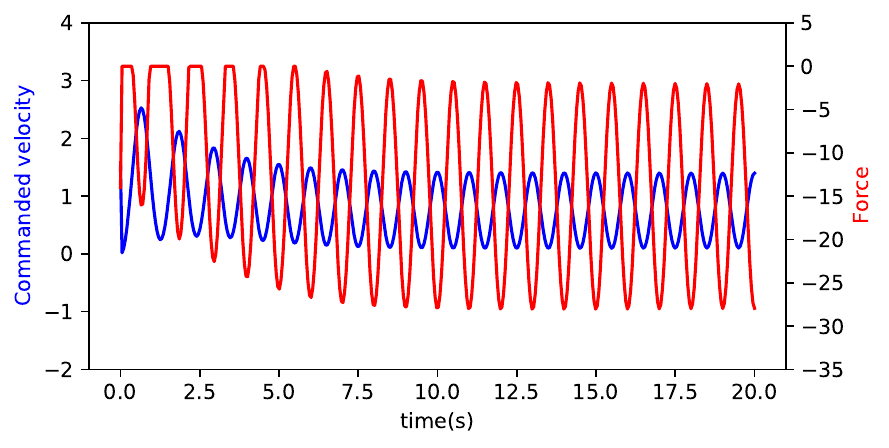}}
\\
\subfloat[With $\cL_2$-gain constraint]{{\label{instability_b}}
\includegraphics[width=0.45\textwidth,trim={0cm 0cm 0cm 0cm},clip]{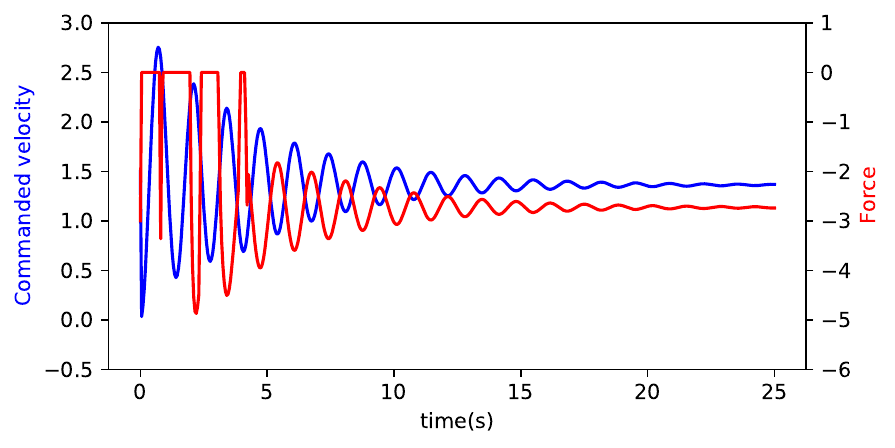}}

\caption{Comparison of stability performance between the condition with small $\cL_2$-gain method and the condition without small $\cL_2$-gain method.}
\label{Fig:instability}
\end{figure}

%% file: instability_human.tex
\begin{figure}[t]
\centering
\vspace{-10pt}
\subfloat[x-direction, without $\cL_2$-gain method]{{\label{instability_a_human}}
\includegraphics[width=0.45\textwidth,trim={0cm 0cm 0cm 0cm},clip]{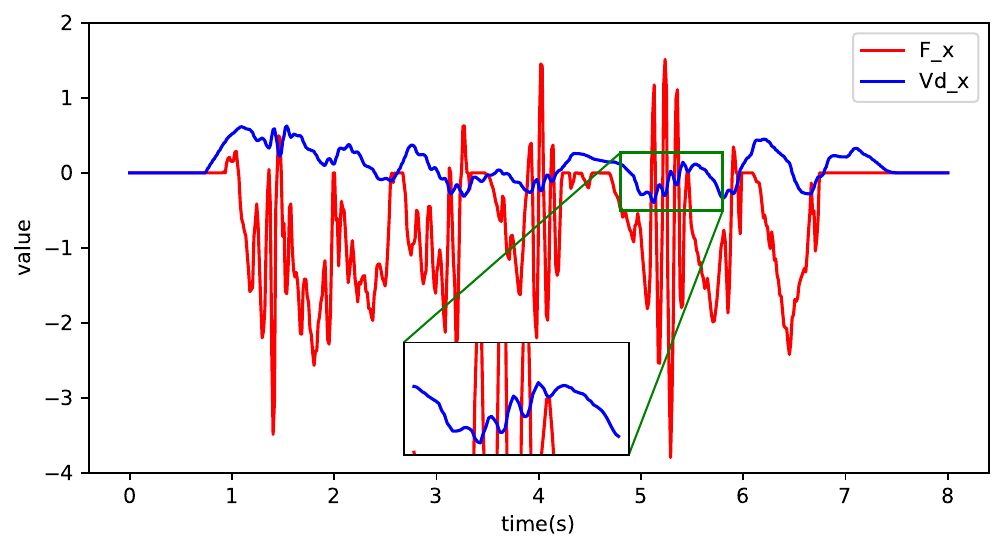}}
\vspace{-10pt}
\\
\subfloat[y-direction, without $\cL_2$-gain method]{{\label{instability_b_human}}
\includegraphics[width=0.45\textwidth,trim={0cm 0cm 0cm 0cm},clip]{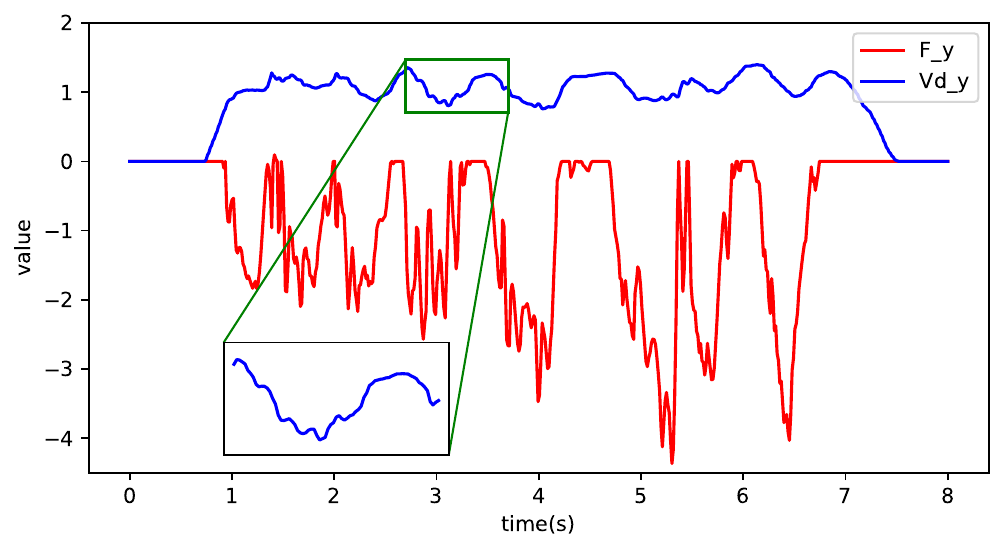}}
\\
\vspace{-10pt}
\subfloat[x-direction, with $\cL_2$-gain method]{{\label{instability_c_human}}
\includegraphics[width=0.45\textwidth,trim={0cm 0cm 0cm 0cm},clip]{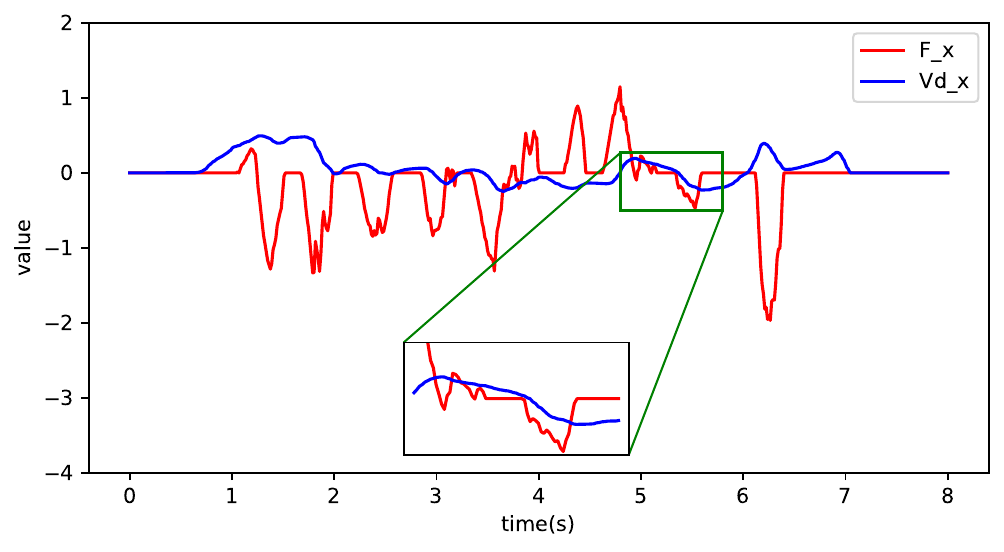}}
\vspace{-10pt}
\\
\subfloat[y-direction, with $\cL_2$-gain method]{{\label{instability_d_human}}
\includegraphics[width=0.45\textwidth,trim={0cm 0cm 0cm 0cm},clip]{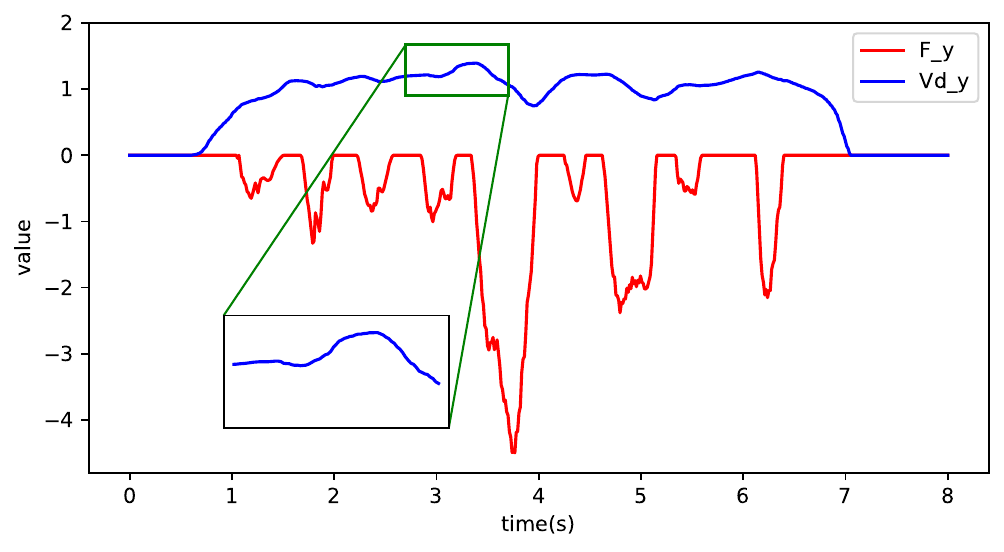}}
\caption{Results of the human's commanded velocity and the force feedback in a two-dimensional simulation.}
\label{Fig:instability_human}
\vspace{-5pt}
\end{figure}

%% file: results_SCF_no_obstacle.tex
\begin{figure*}[!ht]
\vspace{-20pt}
\centering
\subfloat[]
{
\begin{tikzpicture}
\node[inner sep=0pt] (pica)[label ={[rotate=90,left=-4pt,anchor =south]180:\scriptsize $k_v=1$}]{\label{subfig:SCF_a}
\includegraphics[trim={0.5cm 0cm 0cm 0cm,clip},width=0.22\textwidth]{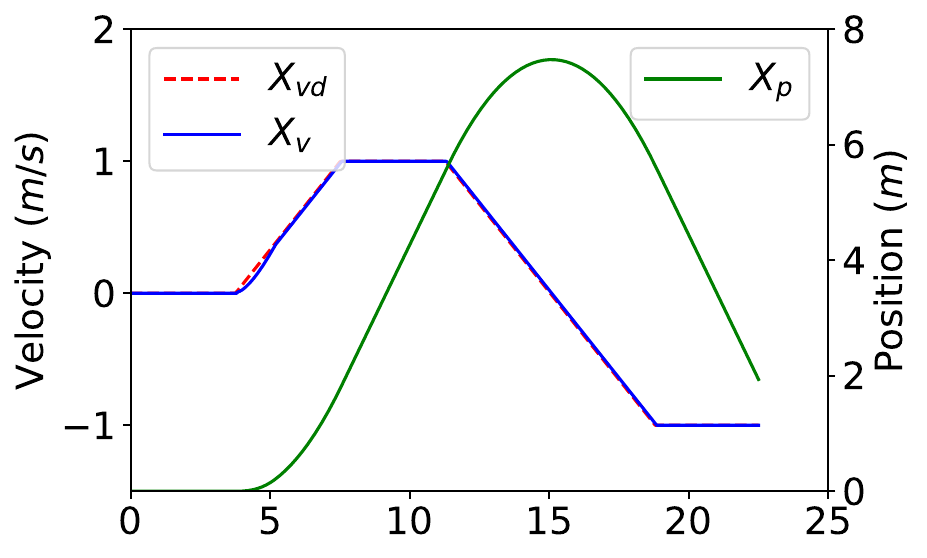}};
\end{tikzpicture}
}
\subfloat[]{\label{subfig:SCF_b}
\includegraphics[width=0.22\textwidth,trim={0.5cm 0cm 0cm 0cm}]{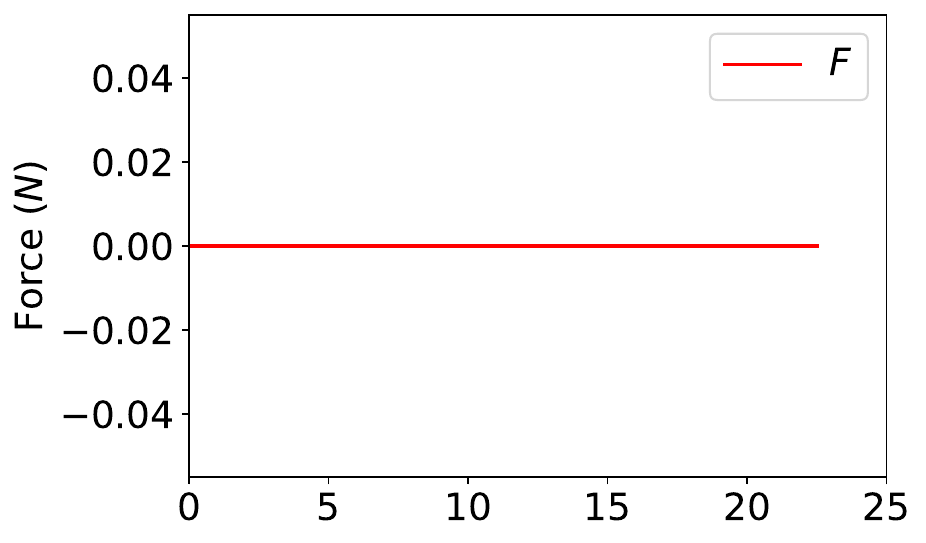}
}
\subfloat[]{\label{subfig:SCF_c}
\includegraphics[width=0.22\textwidth,trim={0cm 0cm 0cm 0cm}]{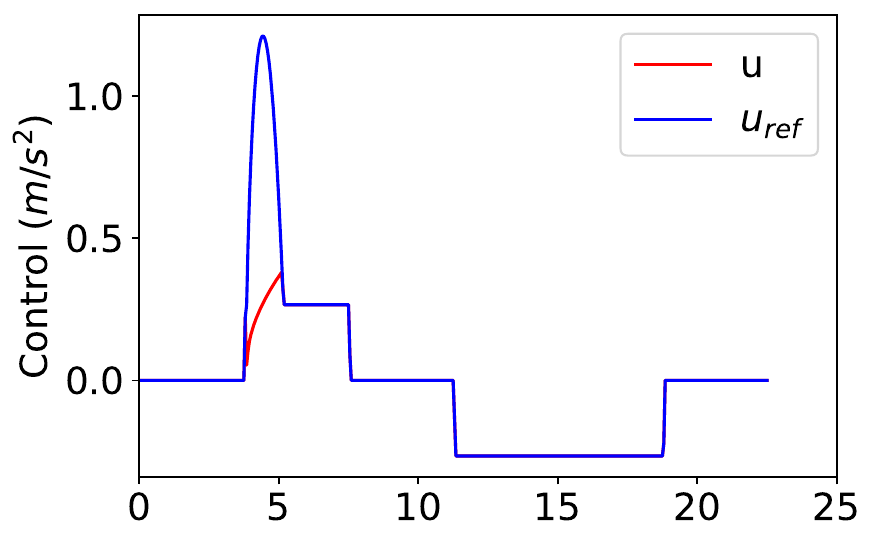}
}
\subfloat[]{
\includegraphics[width=0.22\textwidth,trim={0cm 0cm 0cm 0cm}]{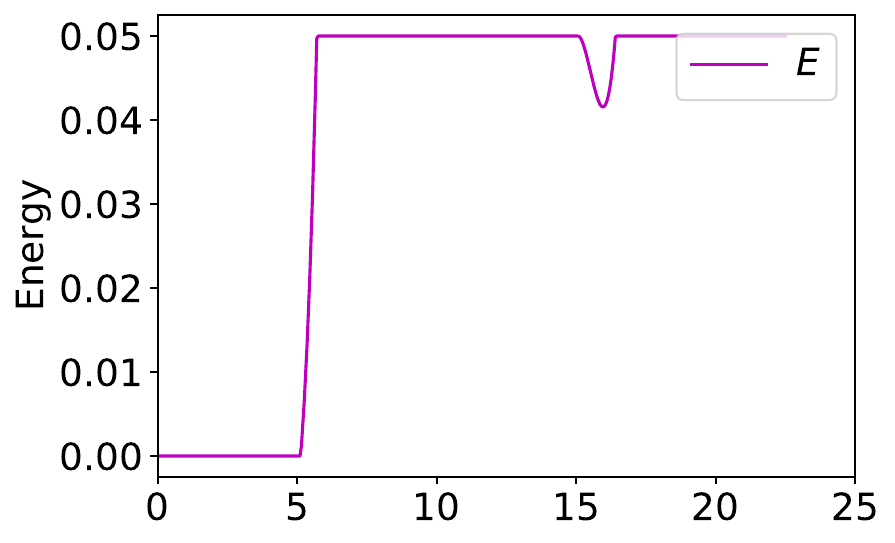}
}
\vspace{-3pt}
\\
\subfloat[]
{
\begin{tikzpicture}
\node[inner sep=0pt] (pica)[label ={[rotate=90,left=-4pt,anchor =south]180:\scriptsize $k_v=5$}]{\label{subfig:SCF_e}
\includegraphics[trim={0.5cm 0cm 0cm 0cm,clip},width=0.22\textwidth]{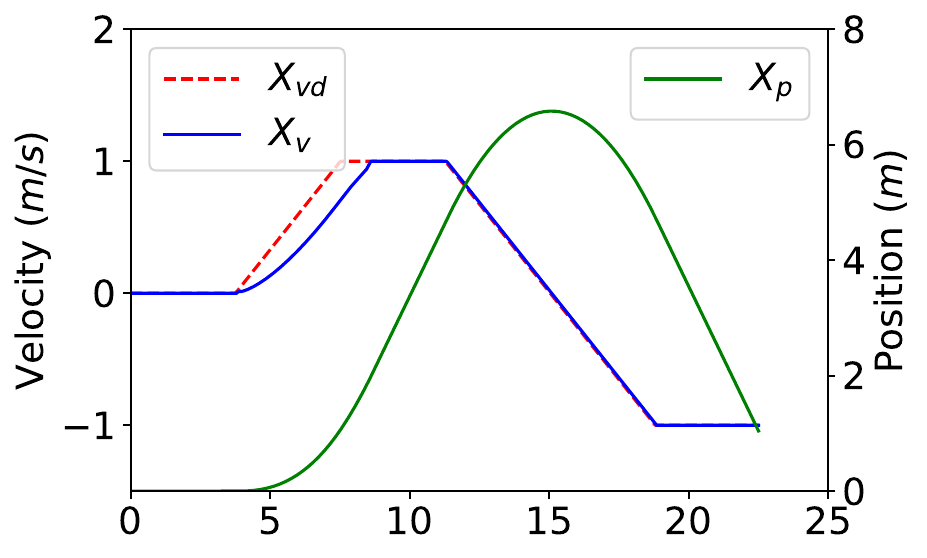}};
\draw[red, thick] (-0.9,-0.2)node[below] {$A$} -- (-0.35,-0.2)--(-0.35, 0.7)--(-0.9, 0.7)--(-0.9,-0.2);
\end{tikzpicture}
}
\subfloat[]{\label{subfig:SCF_f}
\includegraphics[width=0.22\textwidth,trim={0.5cm 0cm 0cm 0cm}]{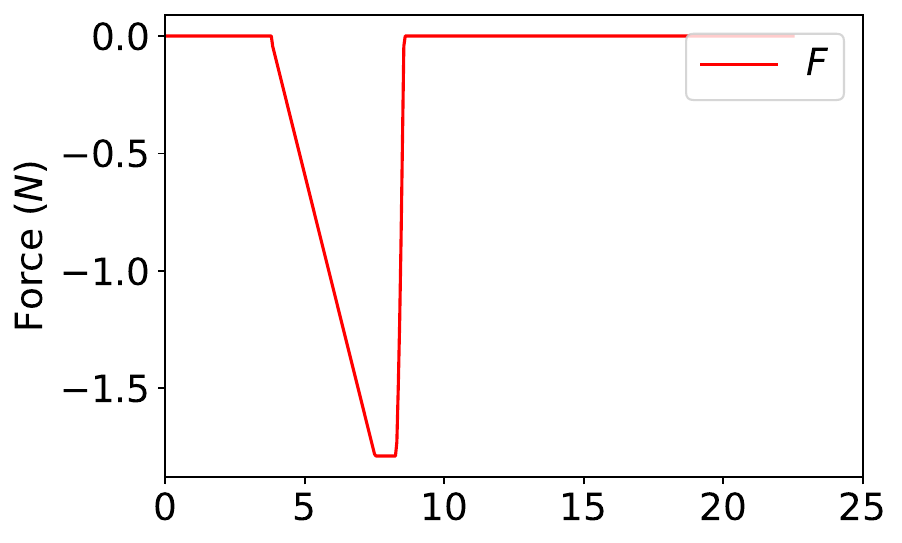}
}
\subfloat[]{
\includegraphics[width=0.22\textwidth,trim={0cm 0cm 0cm 0cm}]{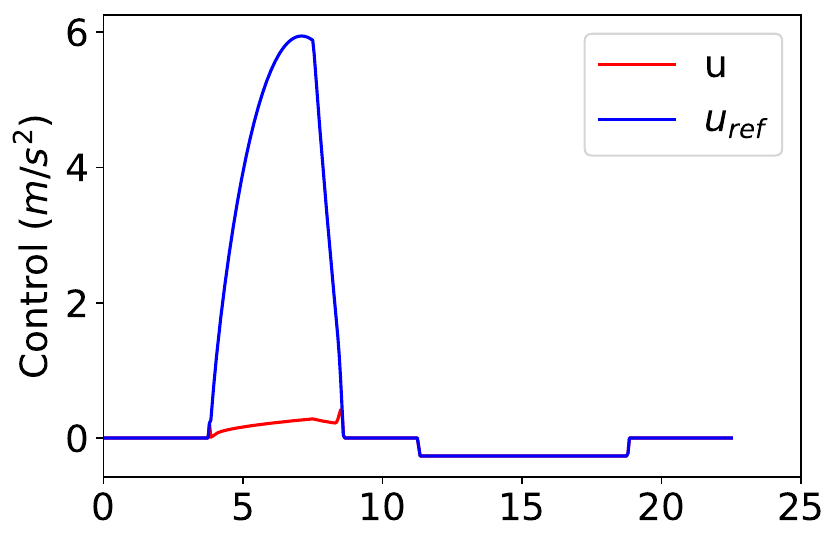}
}
\subfloat[]{
\includegraphics[width=0.22\textwidth,trim={0cm 0cm 0cm 0cm}]{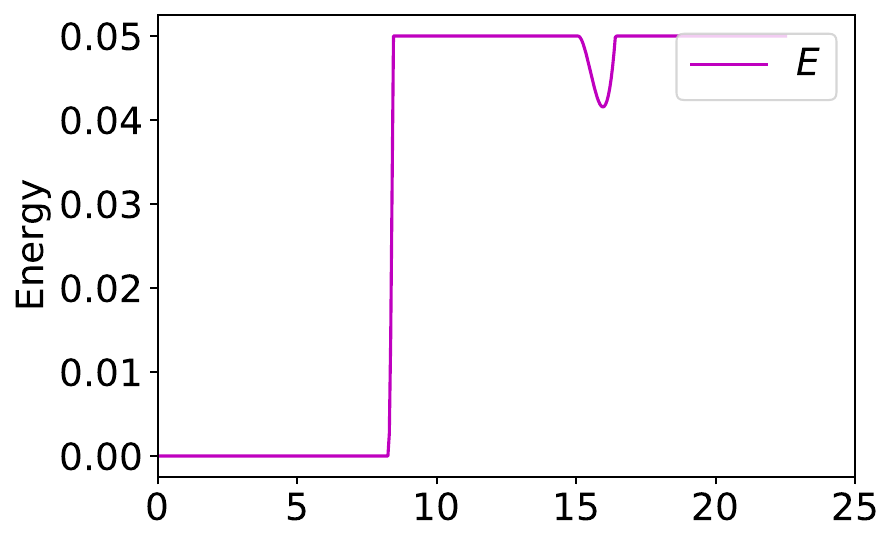}
}
\caption{\small Results of the SCF design when there are no obstacles. The first row shows results with $k_v=1$ while the second row with $k_v=5$. For each column, it respectively shows the state of the UAV, the force feedback $F$, the generated control $u$ and the reference control $u_{\rm{ref}}$, and the energy of the system. }
\label{Fig:SCF_free}
\vspace{-10pt}
\end{figure*}

%% file: results_JCF_no_obstacle.tex
\begin{figure*}[!ht]
\centering
\subfloat[]
{
\begin{tikzpicture}
\node[inner sep=0pt] (pica)[label ={[rotate=90,left=-4pt,anchor =south]180:\scriptsize $k_v=1$}]{\label{subfig:JCF_a}
\includegraphics[trim={0.5cm 0cm 0cm 0cm,clip},width=0.22\textwidth]{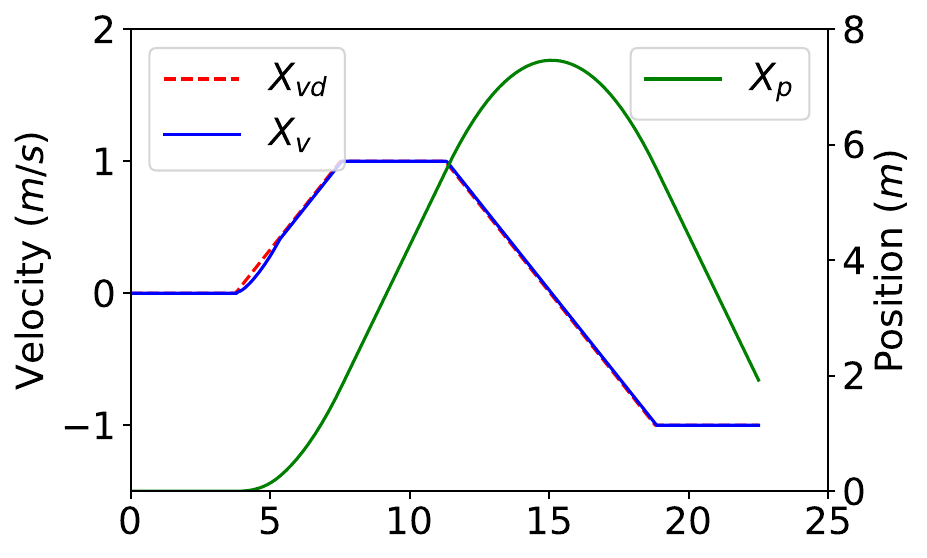}};
\draw[red, thick] (-1.1,-0.2)node[below] {$B$} -- (-0.6,-0.2)--(-0.6, 0.3)--(-1.1, 0.3)--(-1.1,-0.2);
\end{tikzpicture}
}
\subfloat[]{
\label{subfig:JCF_b}
\includegraphics[width=0.22\textwidth,trim={0.5cm 0cm 0cm 0cm}]{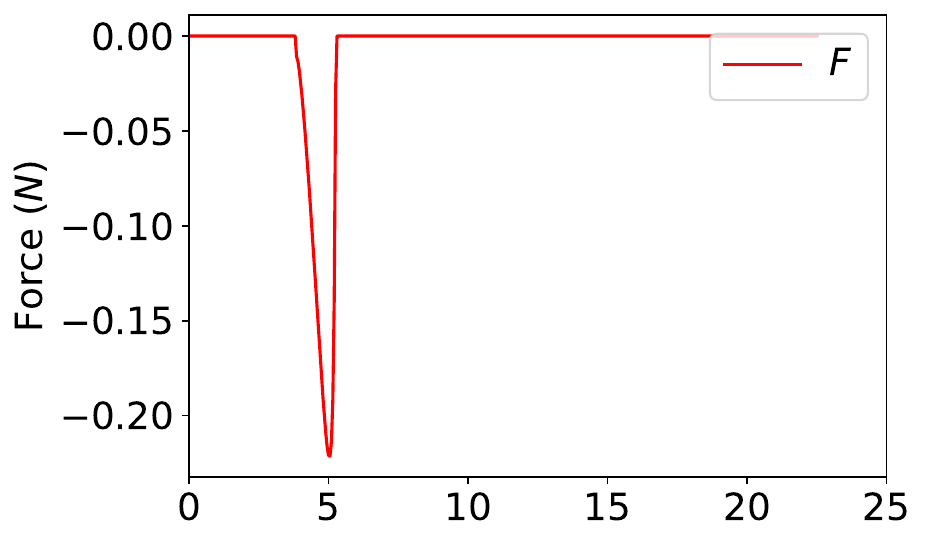}
}
\subfloat[]{\label{subfig:JCF_c}
\includegraphics[width=0.22\textwidth,trim={0cm 0cm 0cm 0cm}]{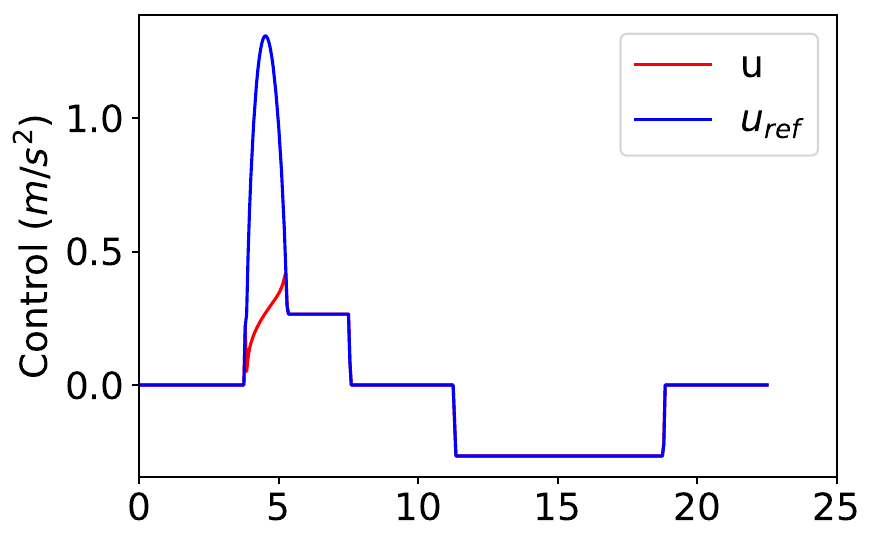}
}
\subfloat[]{\label{subfig:JCF_d}
\includegraphics[width=0.22\textwidth,trim={0cm 0cm 0cm 0cm}]{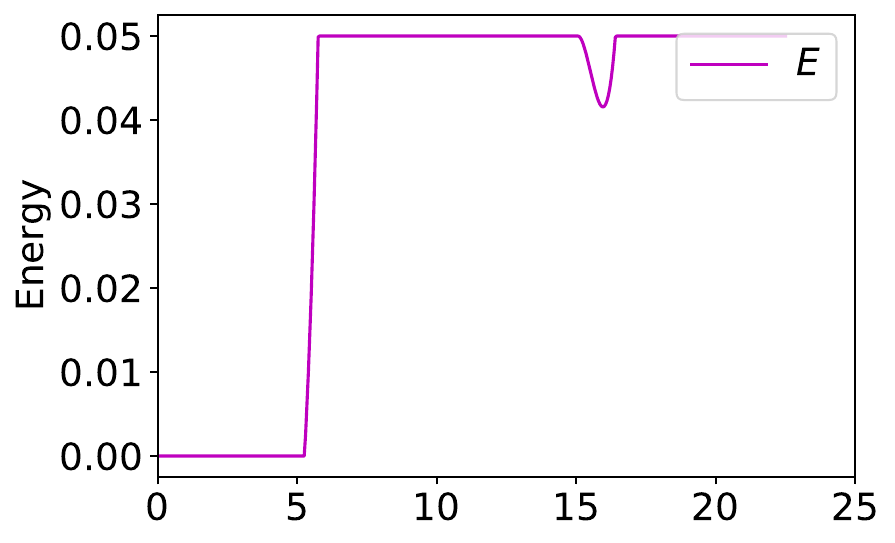}
}
\vspace{-3pt}
\\
\subfloat[]
{
\begin{tikzpicture}
\node[inner sep=0pt] (e)[label ={[rotate=90,left=-4pt,anchor =south]180:\scriptsize $k_v=5$}]{\label{subfig:JCF_e}
\includegraphics[trim={0.5cm 0cm 0cm 0cm,clip},width=0.22\textwidth]{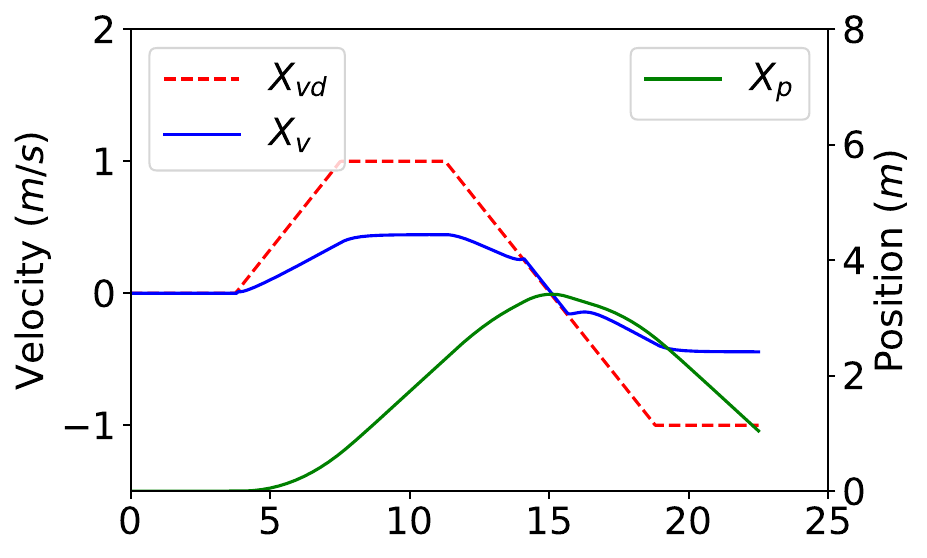}};
\draw[red, thick] (-1,-0.2)node[below] {$C$} -- (0.2,-0.2)--(0.2, 0.7)--(-1, 0.7)--(-1,-0.2);
\end{tikzpicture}
}
\subfloat[]{
\begin{tikzpicture}
\node[inner sep=0pt] (f){\label{subfig:JCF_f}
\includegraphics[trim={0cm 0cm 0cm 0cm,clip},width=0.22\textwidth]{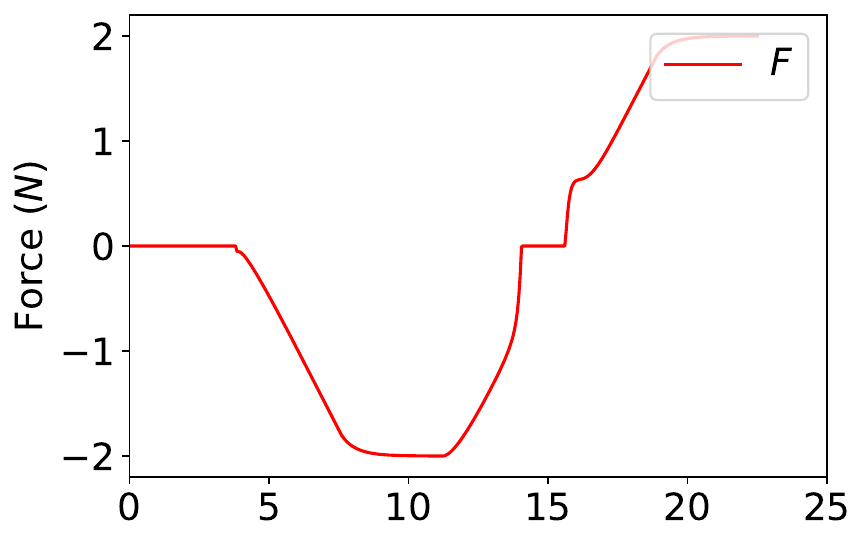}};
\draw[red, thick] (-0.7,-0.8) -- (0.5,-0.8)--(0.5, 0.2)--(-0.7, 0.2)node[above] {$D$}--(-0.7,-0.8);
\end{tikzpicture}

}
\subfloat[]{\label{subfig:JCF_g}
\includegraphics[width=0.22\textwidth,trim={0cm 0cm 0cm
0cm}]{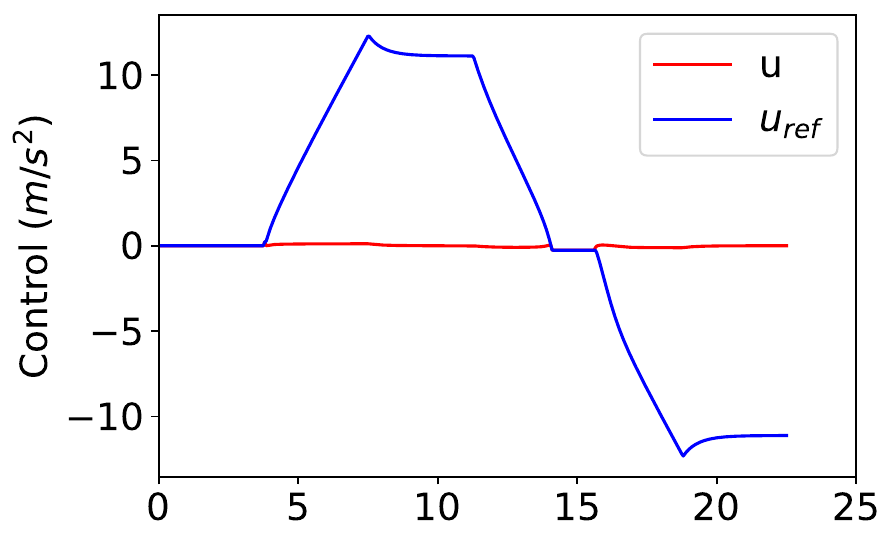}
}
\subfloat[]{\label{subfig:JCF_h}
\includegraphics[width=0.22\textwidth,trim={0cm 0cm 0cm 0cm}]{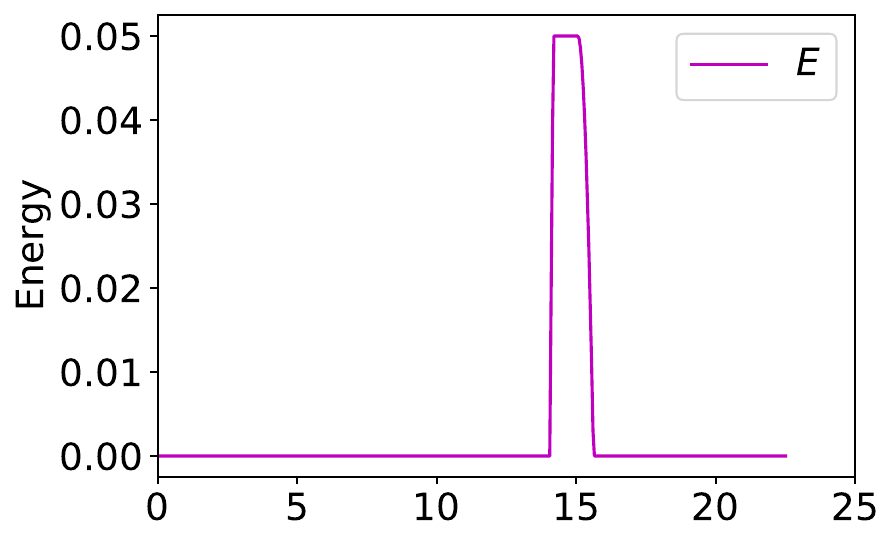}
}
\caption{\small Results of the JCF design when there are no obstacles. The first row shows results with $k_v=1$ while the second row with $k_v=5$. For each column, it respectively shows the state of the UAV, the force feedback $F$, the generated control $u$ and the reference control $u_{\rm{ref}}$, and the energy of the system. }
\label{Fig:JCF_free}
\end{figure*}

%% file: results_SCF_obstacle.tex
\begin{figure*}[t]
\raggedright
\subfloat[]
{
\begin{tikzpicture}
\node[inner sep=0pt] (pica)[label ={[rotate=90,left=-4pt,anchor =south]180:\scriptsize $k_v=1$}]{\label{Fig:SCF_obs_a}
\includegraphics[trim={0.5cm 0cm 0cm 0cm,clip},width=0.22\textwidth]{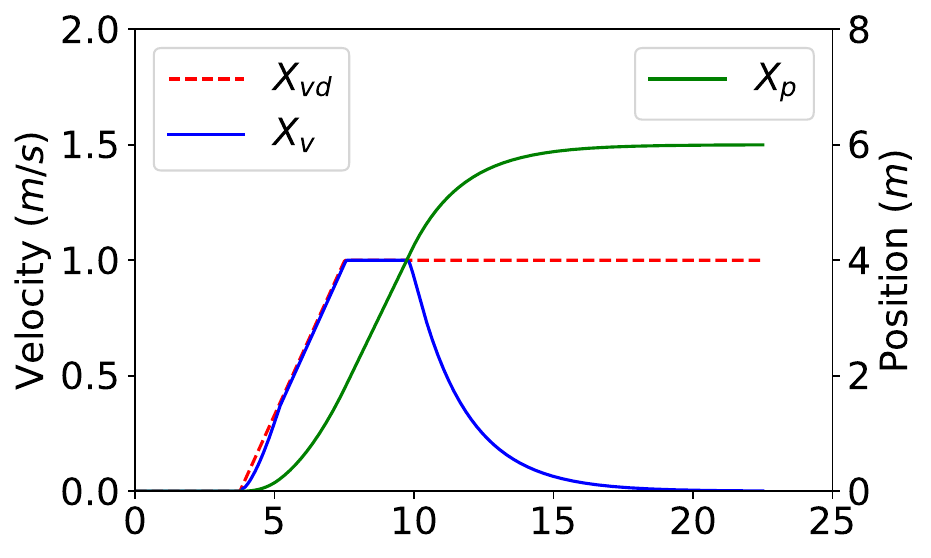}};
\draw[red, thick] (-0.2,0.3) -- (0.9,0.3)node[below] {$E$}--(0.9, 0.7)--(-0.2, 0.7)--(-0.2, 0.3);
\end{tikzpicture}
}
\subfloat[]{
\begin{tikzpicture}
\node[inner sep=0pt] (b){\label{Fig:SCF_obs_b}
\includegraphics[trim={0.5cm 0cm 0cm 0cm,clip},width=0.22\textwidth]{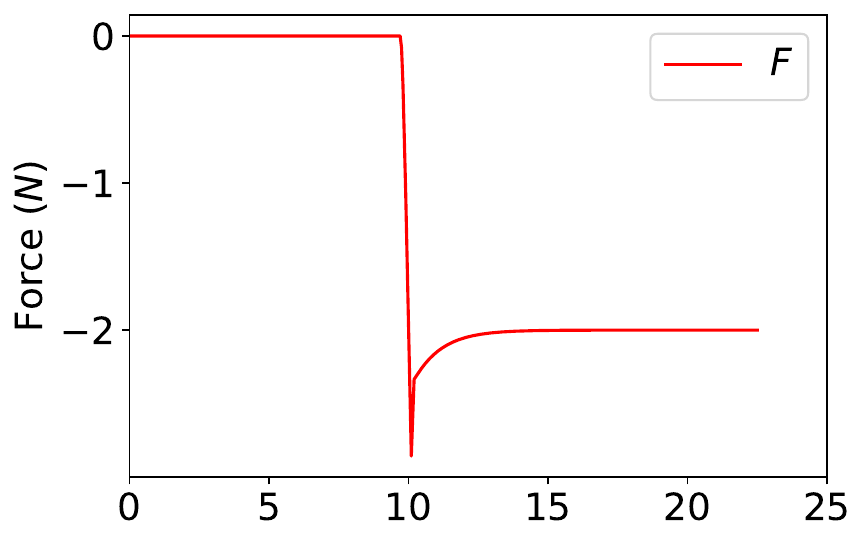}};
\draw[red, thick] (-0.3,-0.8) -- (0.4,-0.8)--(0.4, -0.1)node[above] {$F$}--(-0.3, -0.1)--(-0.3,-0.8);
\end{tikzpicture}
}
\subfloat[]{
\includegraphics[width=0.22\textwidth,trim={0cm 0cm 0cm 0cm}]{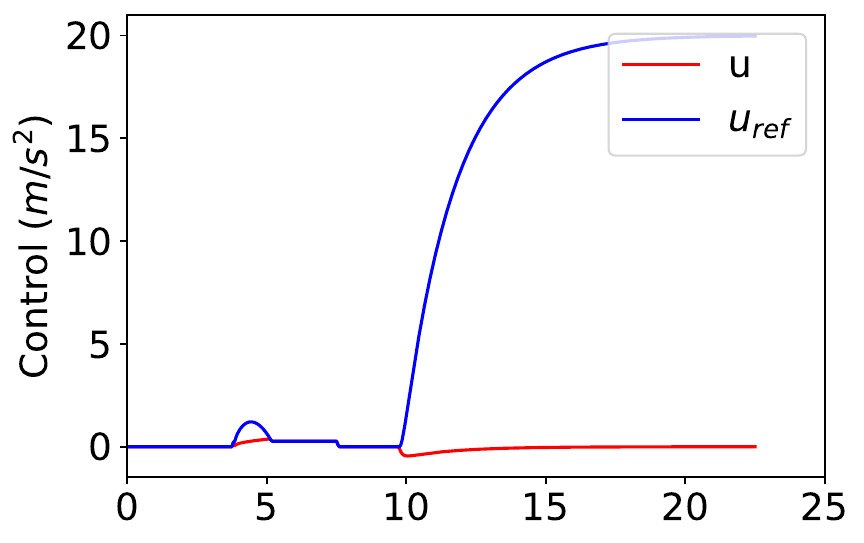}
}
\subfloat[]{
\includegraphics[width=0.22\textwidth,trim={0cm 0cm 0cm 0cm}]{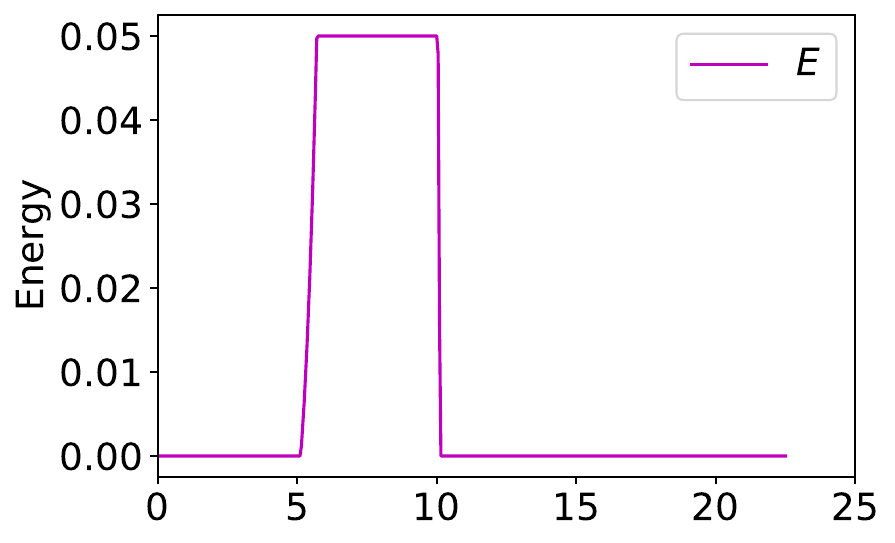}
}
\vspace{-3pt}
\\
\subfloat[]
{
\begin{tikzpicture}
\node[inner sep=0pt] (pica)[label ={[rotate=90,left=-1pt,anchor =south]180:\scriptsize $k_v=5$}]{
\includegraphics[trim={0.5cm 0cm 0cm 0cm,clip},width=0.22\textwidth]{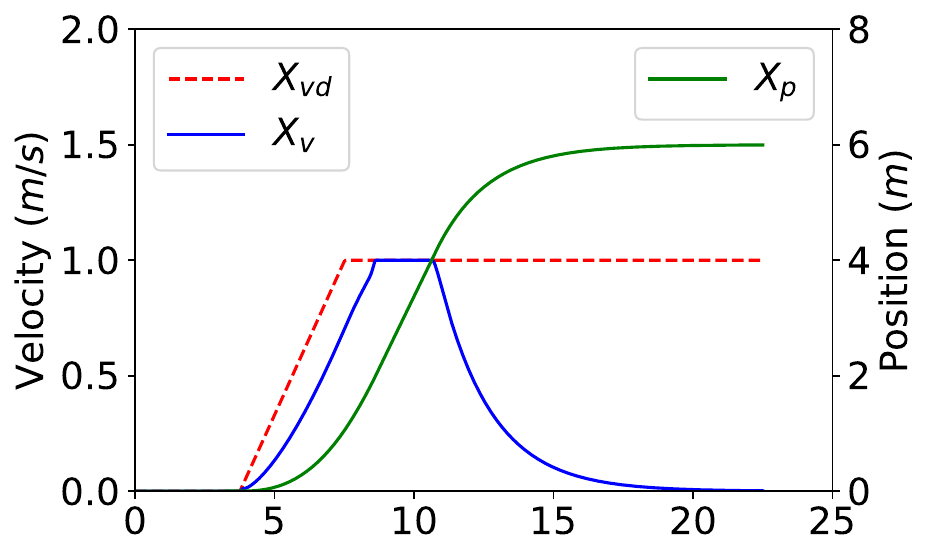}};
\end{tikzpicture}
}
\subfloat[]{
\begin{tikzpicture}
\node[inner sep=0pt] (f){\label{Fig:SCF_obs_f}
\includegraphics[trim={0.5cm 0cm 0cm 0cm,clip},width=0.22\textwidth]{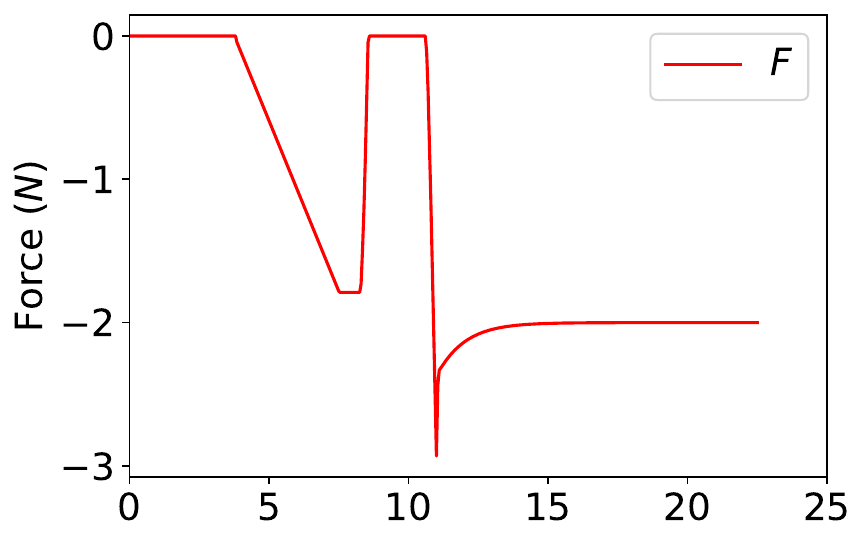}};
\draw[red, thick] (-0.15,-0.8) -- (0.4,-0.8)--(0.4, -0.1)node[above] {$G$}--(-0.15, -0.1)--(-0.15,-0.8);
\end{tikzpicture}
}
\subfloat[]{
\includegraphics[width=0.22\textwidth,trim={0cm 0cm 0cm 0cm}]{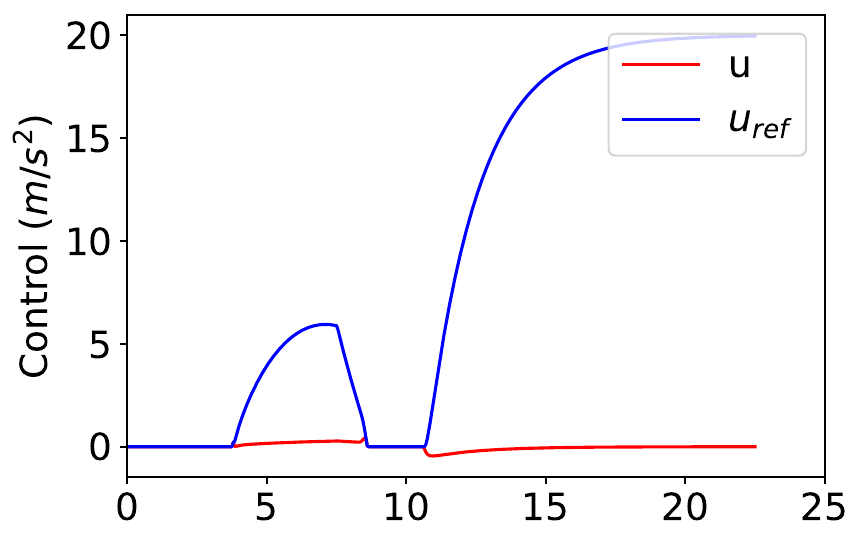}
}
\subfloat[]{
\includegraphics[width=0.22\textwidth,trim={0cm 0cm 0cm 0cm}]{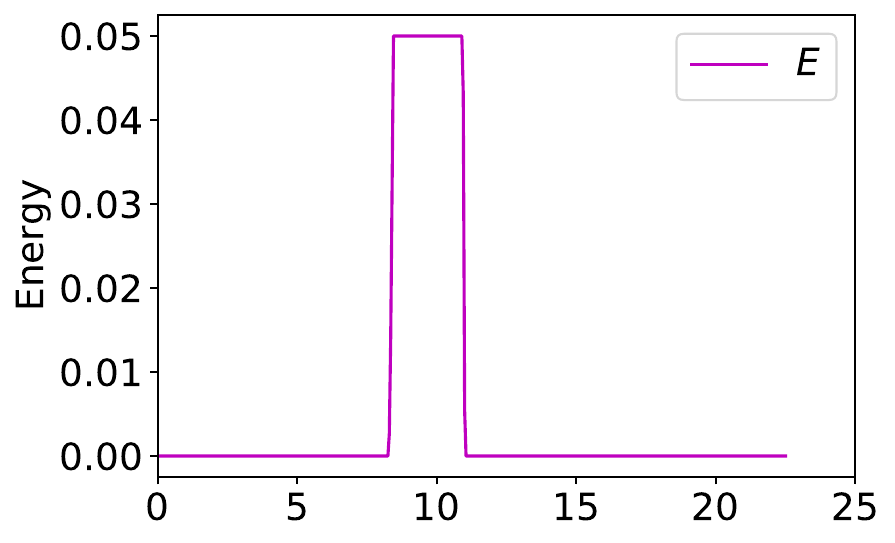}
}
\\
\subfloat[]
{
\begin{tikzpicture}
\node[inner sep=0pt] (pica)[label ={[rotate=90,left=-4pt,anchor =south]180:\scriptsize Passivity}]{\label{Fig:passivity_i}
\includegraphics[trim={0.5cm 0cm 0cm 0cm,clip},width=0.22\textwidth]{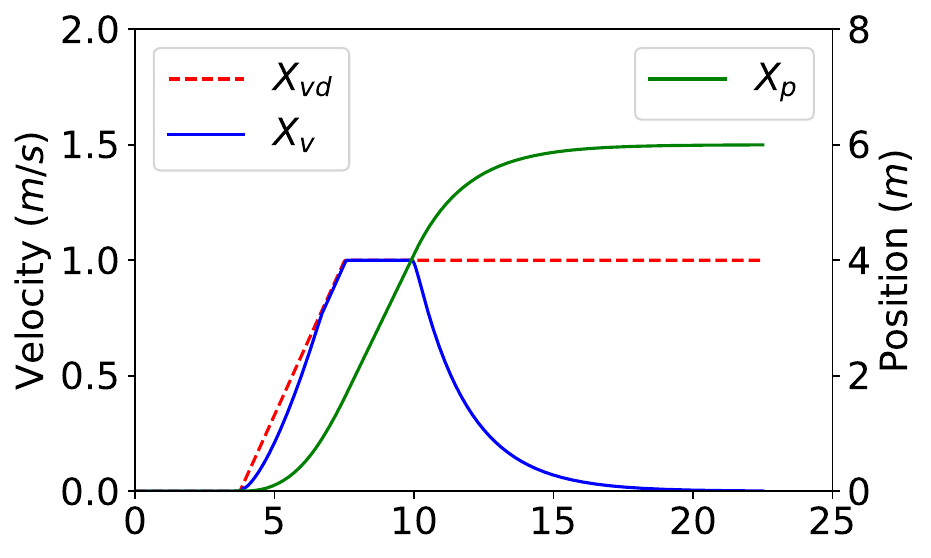}};
\end{tikzpicture}
}
\subfloat[]{
\begin{tikzpicture}
\node[inner sep=0pt]{\label{Fig:passivity_j}
\includegraphics[trim={0.5cm 0cm 0cm 0cm,clip},width=0.22\textwidth]{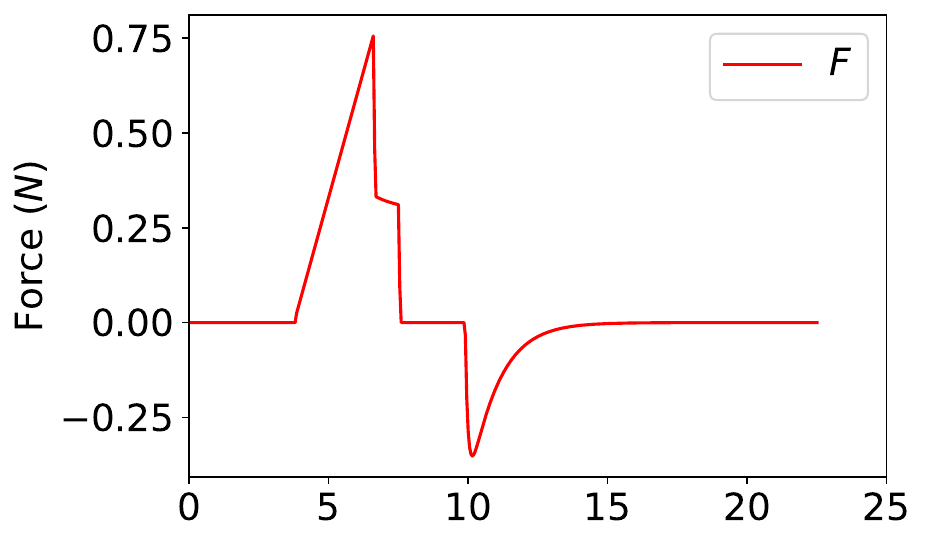}};
\end{tikzpicture}
}
\subfloat[]{\label{Fig:passivity_k}
\includegraphics[width=0.22\textwidth,trim={0cm 0cm 0cm 0cm}]{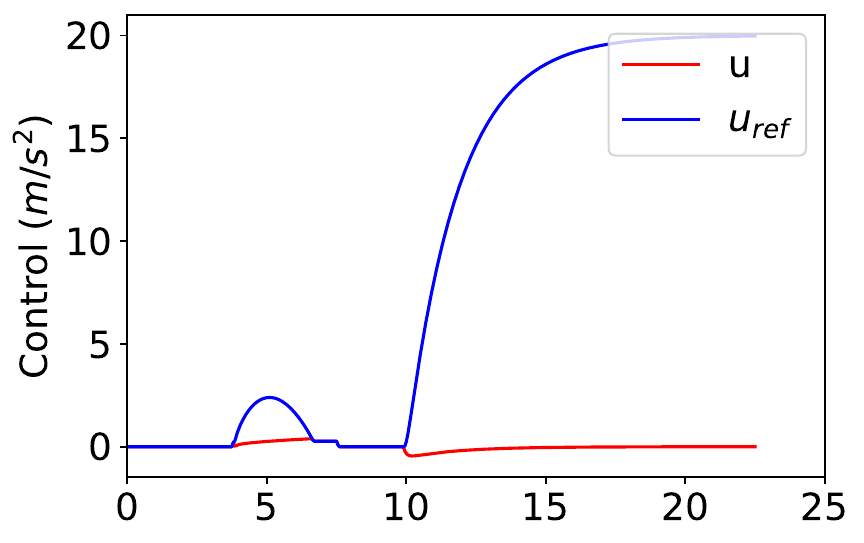}
}

\caption{\small Results of the SCF design when there is a wall at \unit[$6$]{$m$} away. The first row shows results with $k_v=1$ while the second row with $k_v=5$. For each column, it respectively shows the state of the UAV, the force feedback $F$, the generated control $u$ and the reference control $u_{\text{ref}}$, and the energy of the system. The third row shows the results of applying the traditional passivity constraint in our system setup. Note that the passivity constraint does not have the design of the energy tank, but it is still comparable with the $\cL_2$ constraint with $E_{max}=0$.}
\label{Fig:SCF_obs}
\vspace{-15pt}
\end{figure*}

%% file: results_JCF_obstacle.tex
\begin{figure*}[!ht]
\centering
\subfloat[]
{
\begin{tikzpicture}\label{Fig:JCF_obs_a}
\node[inner sep=0pt] (pica)[label ={[rotate=90,left=-1pt,anchor =south]180:\scriptsize $k_v=1$}]{
\includegraphics[trim={0.5cm 0cm 0cm 0cm,clip},width=0.22\textwidth]{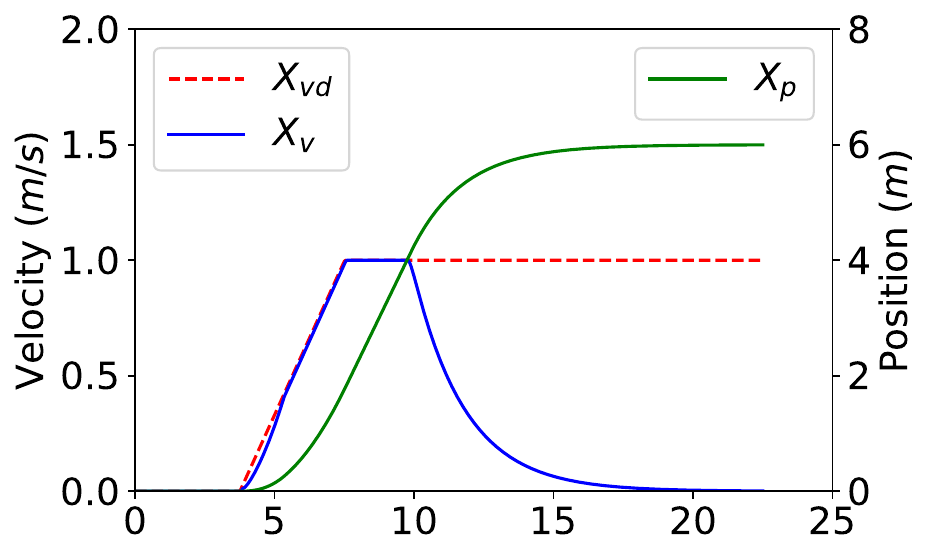}};
\draw[red, thick] (-0.2,0.3) -- (0.9,0.3)node[below] {$H$}--(0.9, 0.7)--(-0.2, 0.7)--(-0.2, 0.3);
\end{tikzpicture}
}
\subfloat[]{
\begin{tikzpicture}
\node[inner sep=0pt] (b){\label{Fig:JCF_obs_b}
\includegraphics[trim={0.5cm 0cm 0cm 0cm,clip},width=0.22\textwidth]{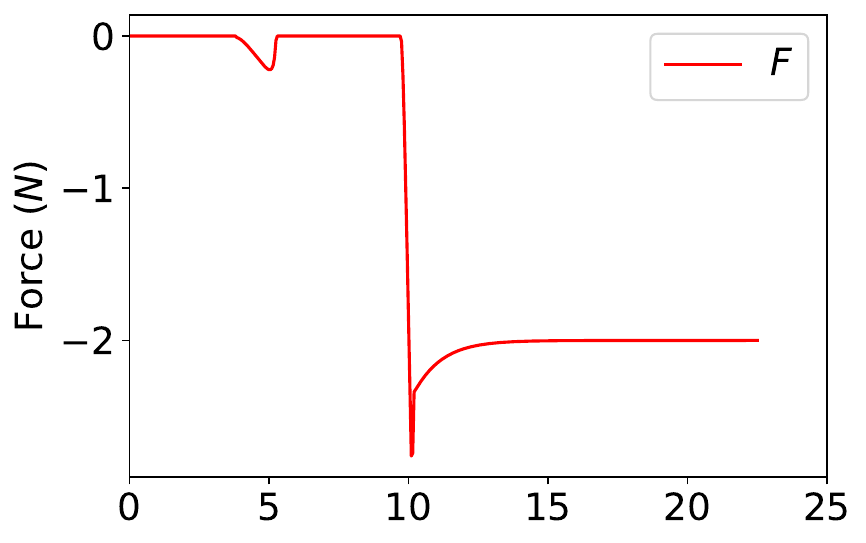}};
\draw[red, thick] (-0.3,-0.8) -- (0.4,-0.8)--(0.4, -0.1)node[above] {$I$}--(-0.3, -0.1)--(-0.3,-0.8);
\end{tikzpicture}
}
\subfloat[]{
\includegraphics[width=0.22\textwidth,trim={0cm 0cm 0cm 0cm}]{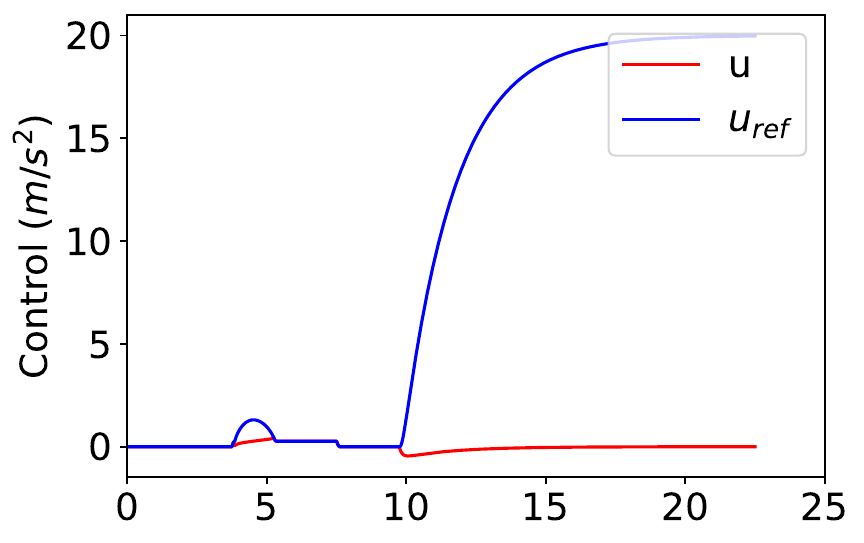}
}
\subfloat[]{
\includegraphics[width=0.22\textwidth,trim={0cm 0cm 0cm 0cm}]{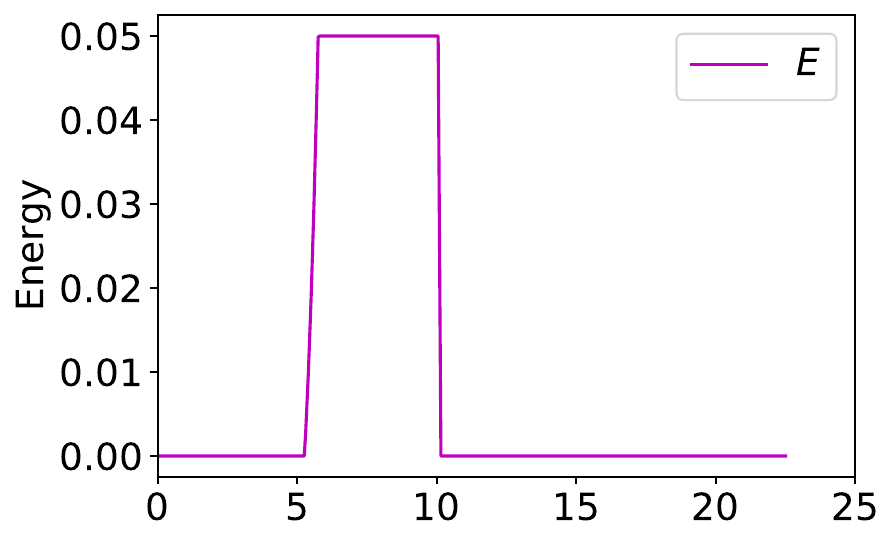}
}
\vspace{-3pt}
\\
\subfloat[]
{
\begin{tikzpicture}
\node[inner sep=0pt] (pica)[label ={[rotate=90,left=-1pt,anchor =south]180:\scriptsize $k_v=5$}]{
\includegraphics[trim={0.5cm 0cm 0cm 0cm,clip},width=0.22\textwidth]{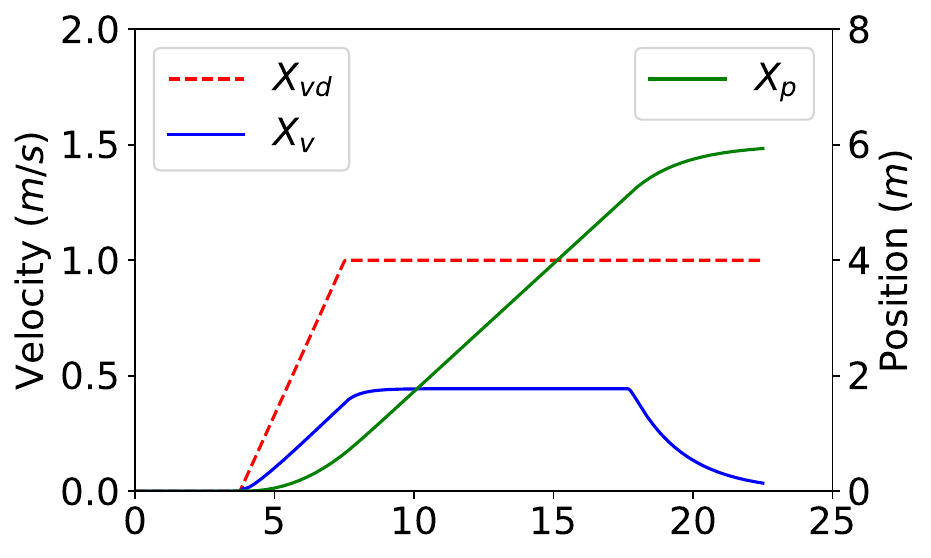}};
\end{tikzpicture}
}
\subfloat[]{
\begin{tikzpicture}
\node[inner sep=0pt]{\label{Fig:JCF_obs_f}
\includegraphics[trim={0.5cm 0cm 0cm 0cm,clip},width=0.22\textwidth]{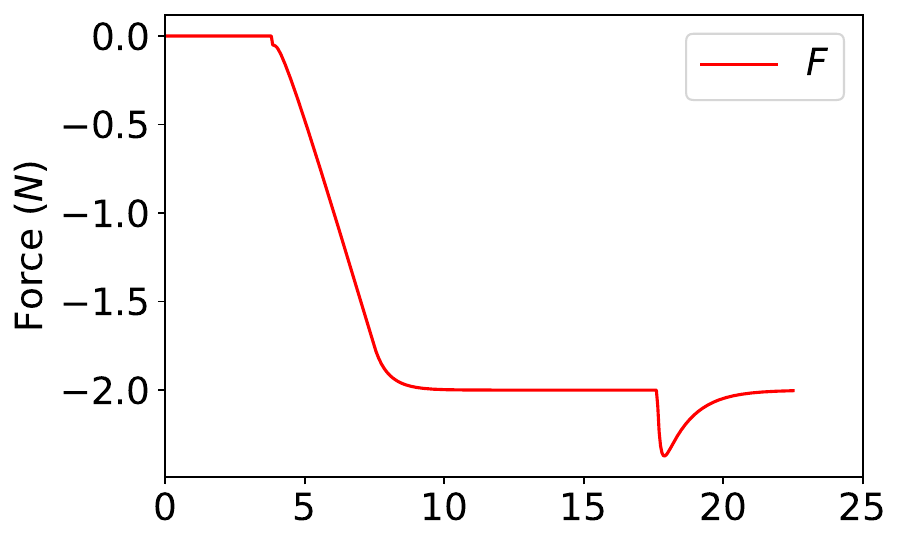}};
\draw[red, thick] (0.6,-0.8) -- (1.1,-0.8)--(1.1, -0.3)node[above] {$J$}--(0.6, -0.3)--(0.6,-0.8);
\end{tikzpicture}
}
\subfloat[]{
\includegraphics[width=0.22\textwidth,trim={0cm 0cm 0cm 0cm}]{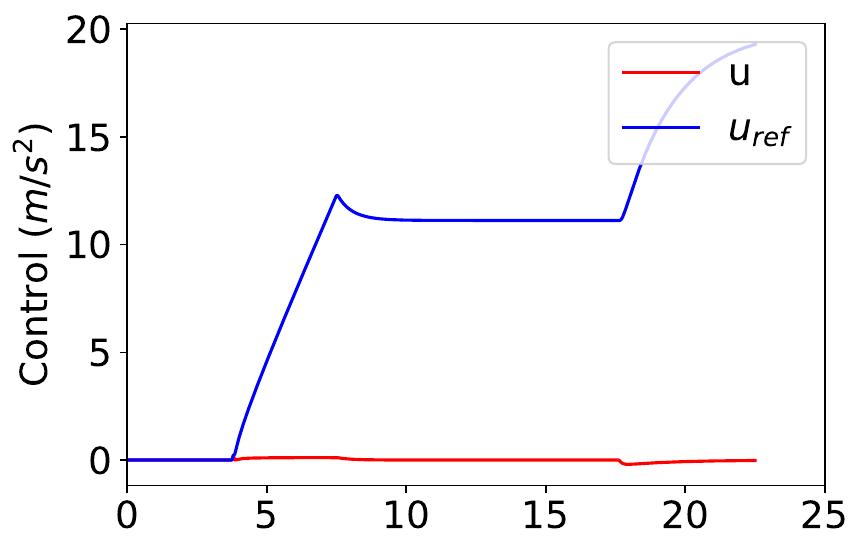}
}
\subfloat[]{
\includegraphics[width=0.22\textwidth,trim={0cm 0cm 0cm 0cm}]{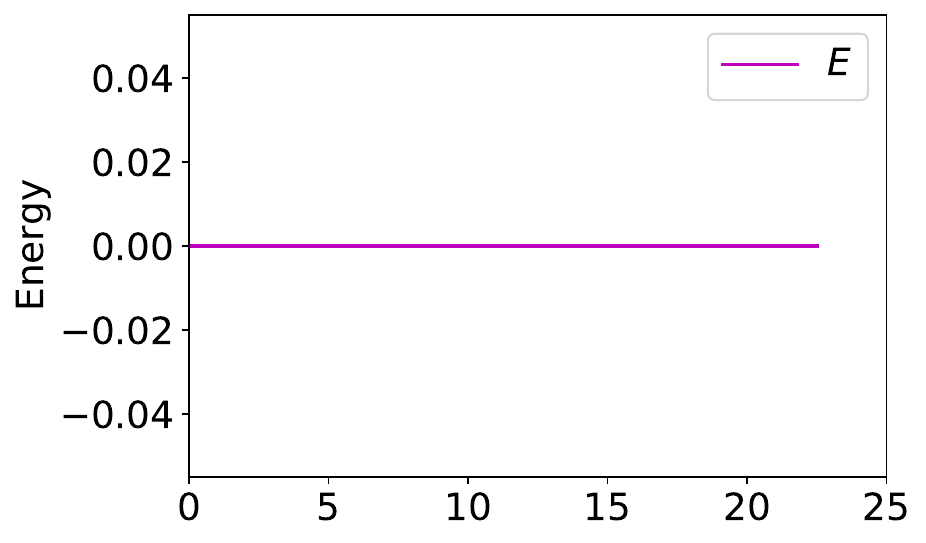}
}
\caption{\small Results of the JCF design when there is a wall at \unit[$6$]{$m$} away. The first row shows results with $k_v=1$ while the second row with $k_v=5$. For each column, it respectively shows the state of the UAV, the force feedback $F$, the generated control $u$ and
the reference control $u_{\text{ref}}$, and the energy of the system. }
\label{Fig:JCF_F_obs}
\vspace{-5pt}
\end{figure*}

%% file: results_diff_Emax.tex
\begin{figure*}[!ht]
\centering
\subfloat[\scriptsize]
{
\begin{tikzpicture}
\node[inner sep=0pt] (pica){
\includegraphics[trim={0cm 0cm 0cm 0cm,clip},width=0.23\textwidth]{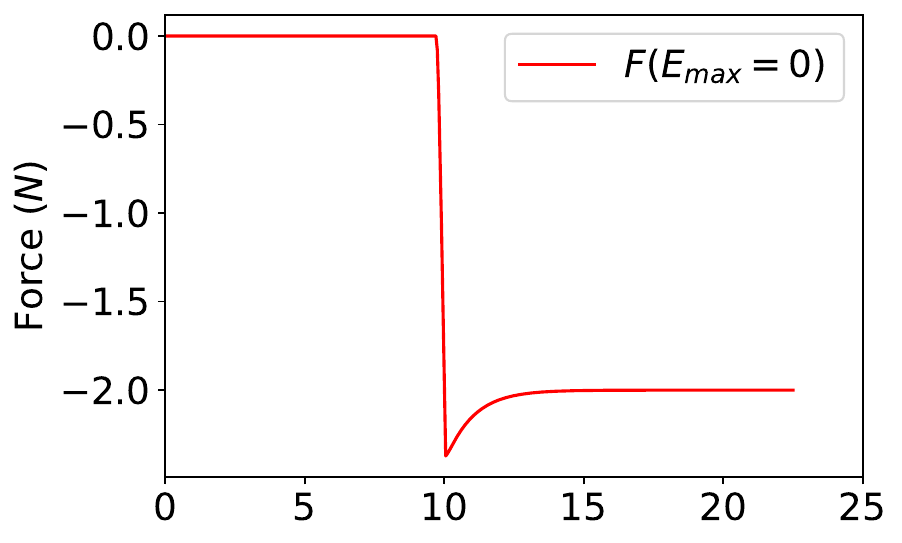}};
\end{tikzpicture}
}
\subfloat[\scriptsize]{
\includegraphics[width=0.23\textwidth,trim={0cm 0cm 0cm 0cm}]{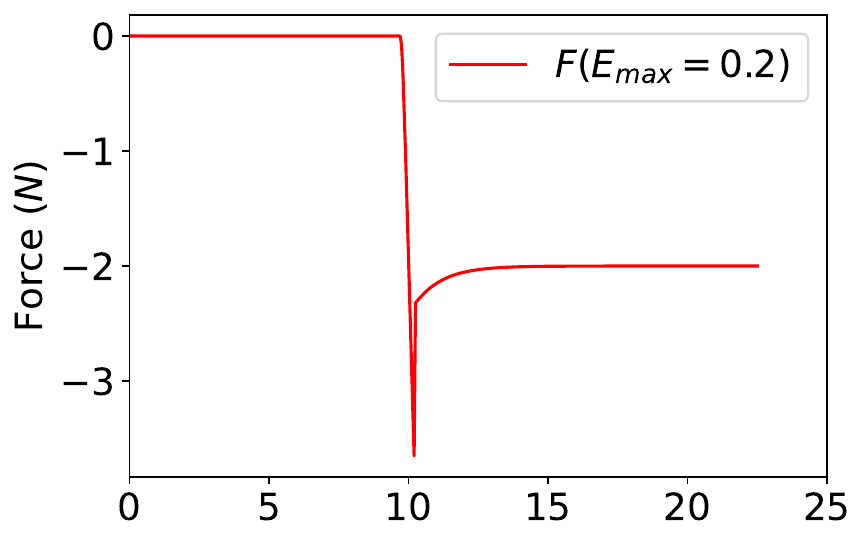}
}
\subfloat[\scriptsize]
{
\begin{tikzpicture}
\node[inner sep=0pt] (pica){
\includegraphics[trim={0cm 0cm 0cm 0cm,clip},width=0.23\textwidth]{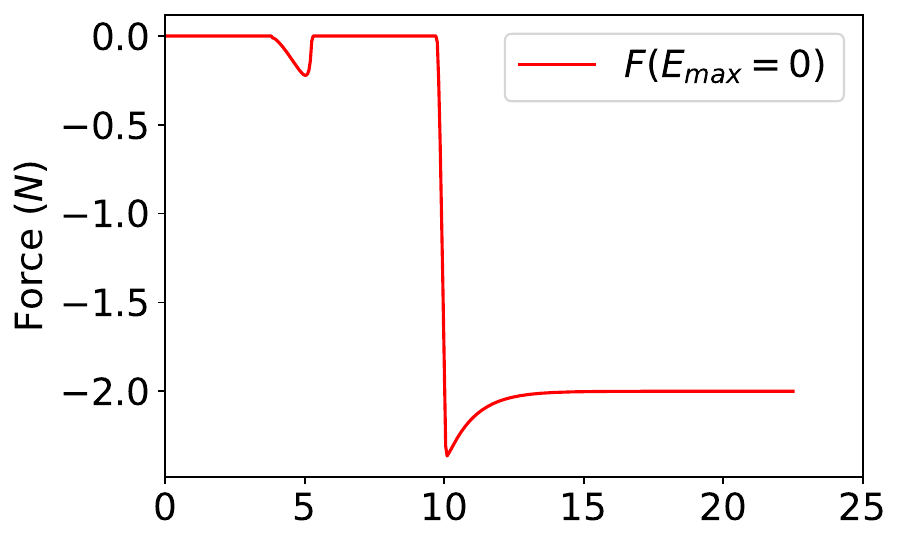}};
\end{tikzpicture}
}
\subfloat[\scriptsize]{
\includegraphics[width=0.23\textwidth,trim={0cm 0cm 0cm 0cm}]{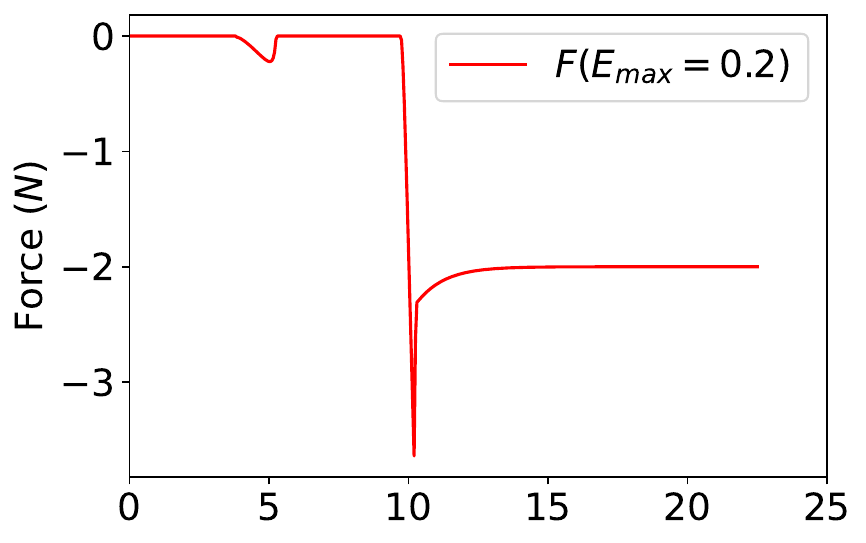}
}
\caption{\small Force feedback $F$ with different maximum energy limits $E_{max}$ (0 or 0.2) when there is an obstacle. (a) and (b) are results from SCF, (c) and (d) are from the JCF design.}
\label{Fig:Force_energy}
\end{figure*}

%% file: results_2D.tex
\begin{figure*}[!ht]
\centering
\subfloat[\scriptsize SCF, $k_v = 1$]{\label{Fig:SCF_Kf_1}
\begin{tikzpicture}
\node[inner sep=0pt] (pica){
\includegraphics[trim={0cm 0cm 0cm 0cm,clip},width=0.45\textwidth]{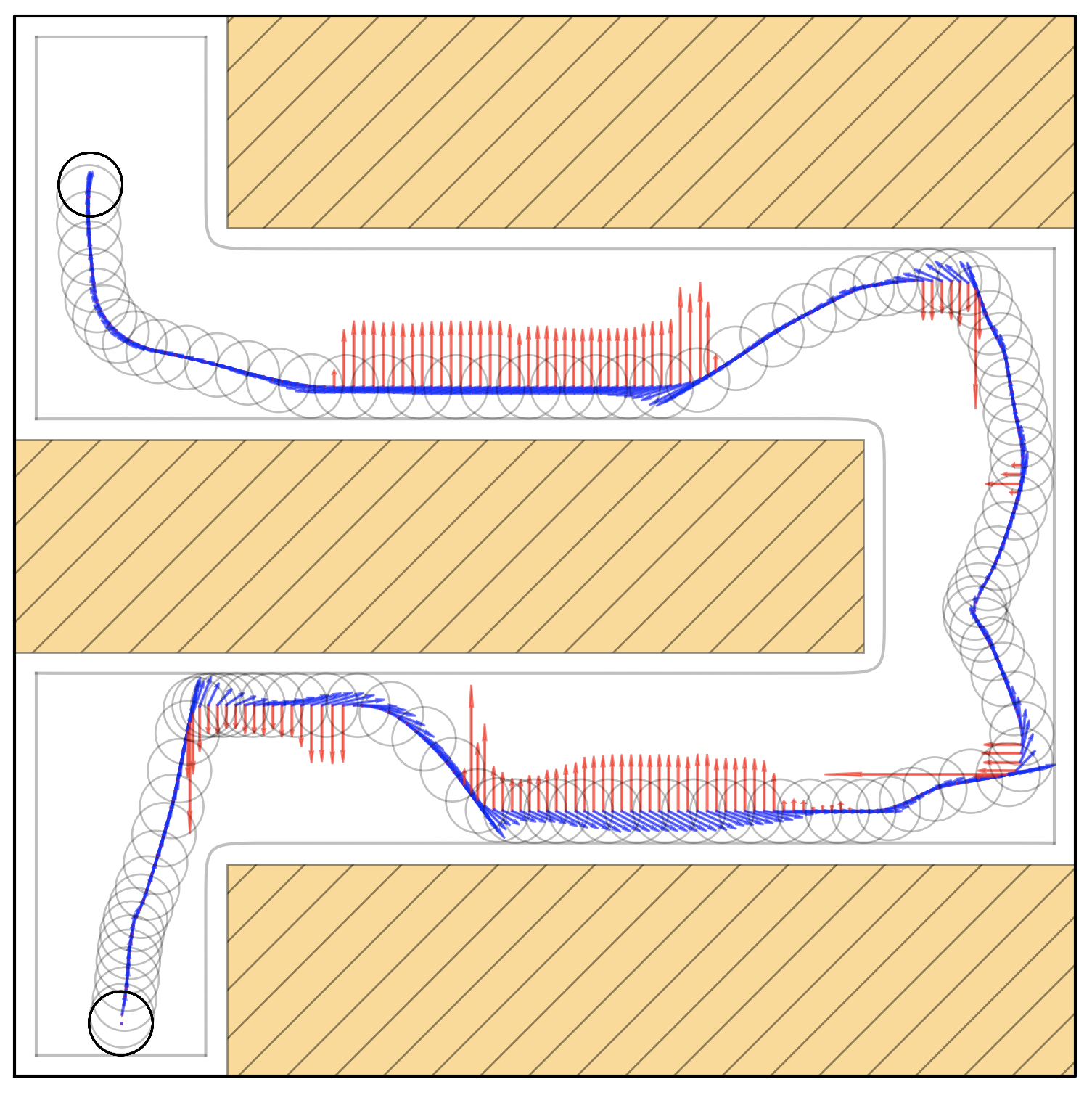}};
\draw[red, thick] (-0.2,-2) -- (1.2,-2)--(1.2, -1.3)node[above] {$K$}--(-0.2, -1.3)--(-0.2,-2);
\end{tikzpicture}
}
\subfloat[\scriptsize SCF, $k_v = 2$]{\label{Fig:SCF_Kf_2}
\includegraphics[width=0.45\textwidth,trim={0cm 0cm 0cm 0cm}]{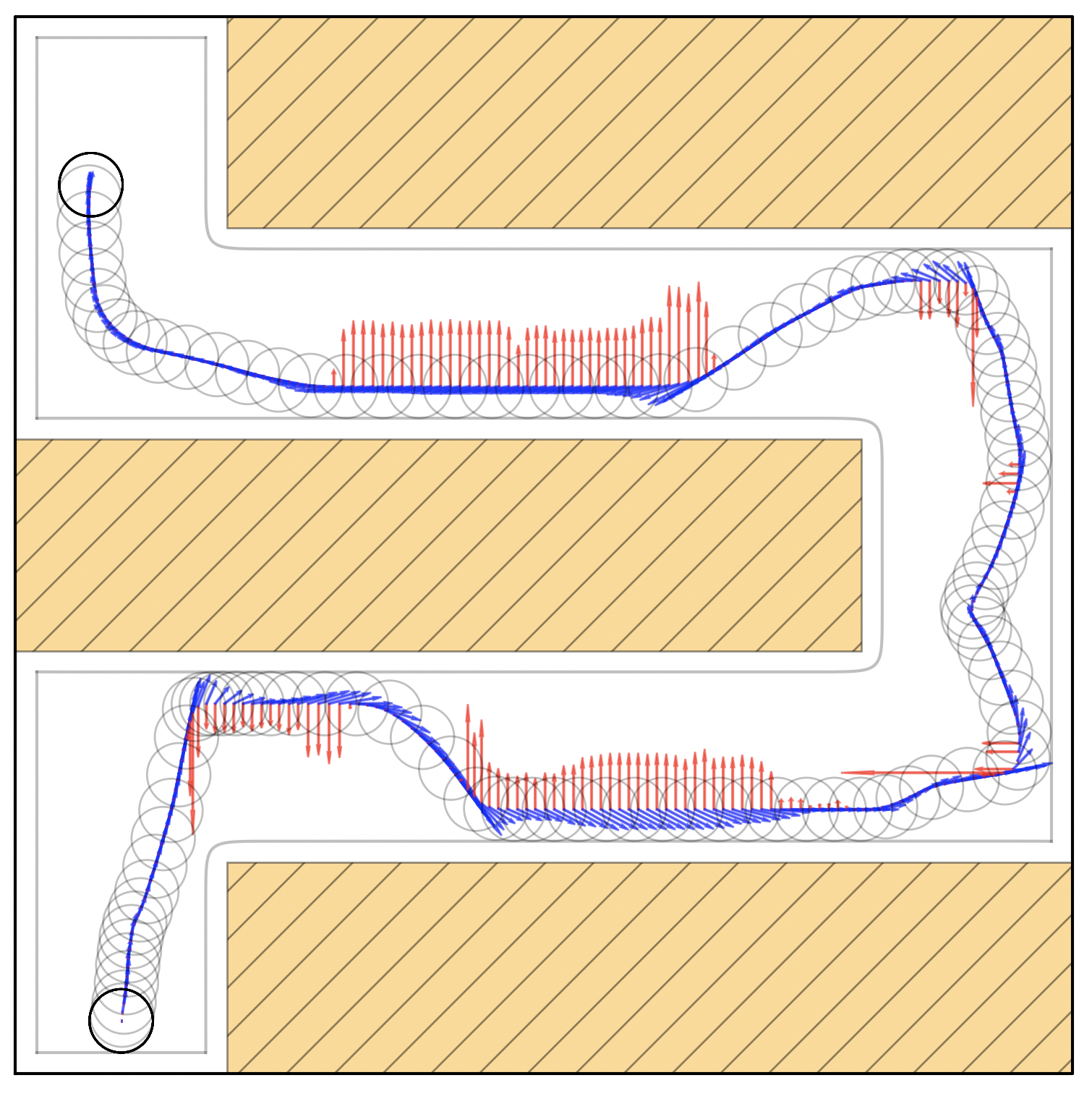}
}\\
\subfloat[\scriptsize JCF, $k_v = 1$]{\label{Fig:JCF_Kf_1}
\begin{tikzpicture}
\node[inner sep=0pt] (pica){
\includegraphics[trim={0cm 0cm 0cm 0cm,clip},width=0.45\textwidth]{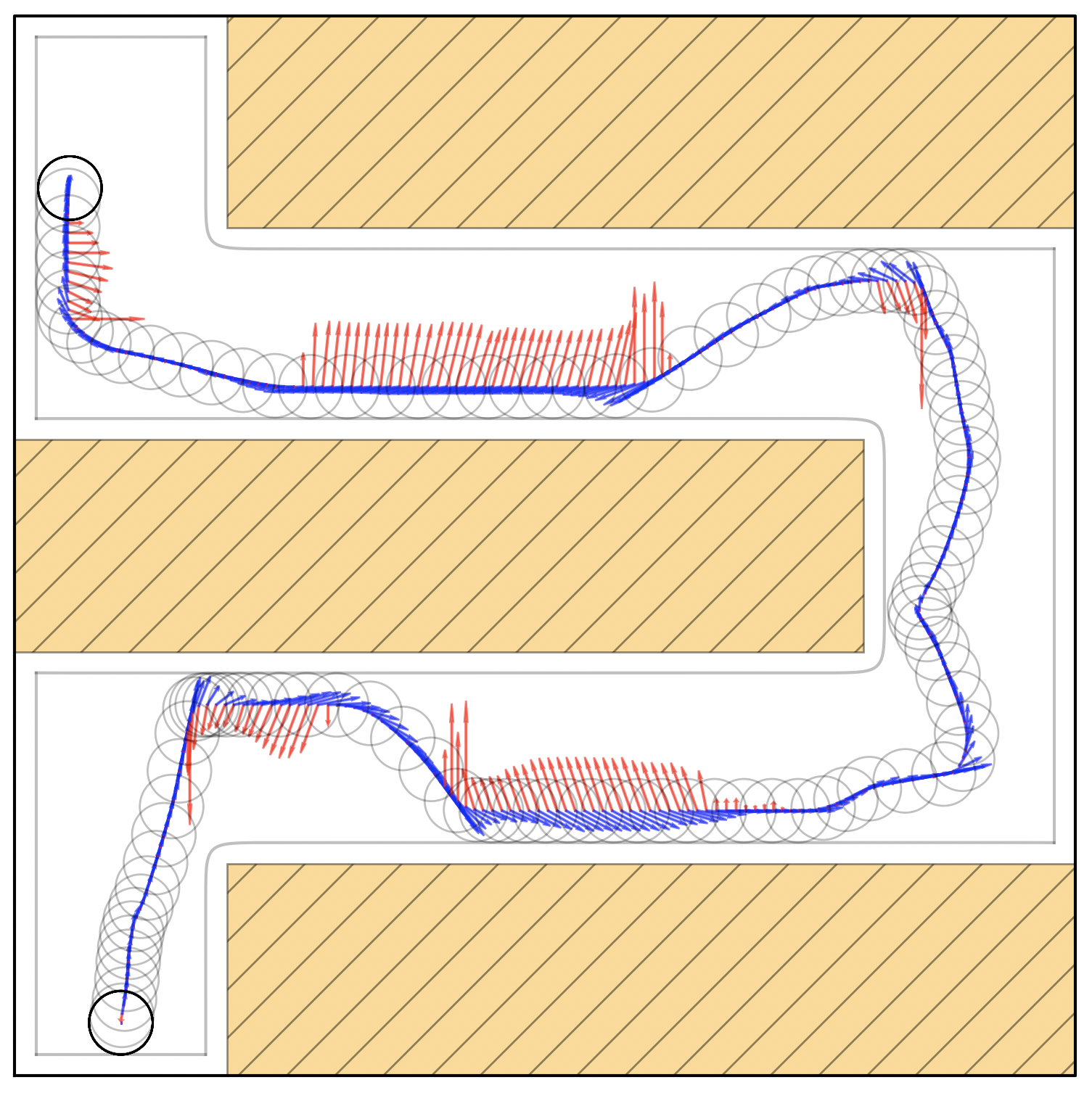}};
\draw[red, thick] (-0.2,-2) -- (1.2,-2)--(1.2, -1.3)node[above] {$L$}--(-0.2, -1.3)--(-0.2,-2);
\draw[red,thick,dashed](2.3,1.5) circle (0.7)node[right, xshift=0.6cm] {$N$};
\end{tikzpicture}
}
\subfloat[\scriptsize JCF, $k_v = 2$]{\label{Fig:JCF_Kf_2}
\begin{tikzpicture}
\node[inner sep=0pt] (pica){
\includegraphics[trim={0cm 0cm 0cm 0cm,clip},width=0.45\textwidth]{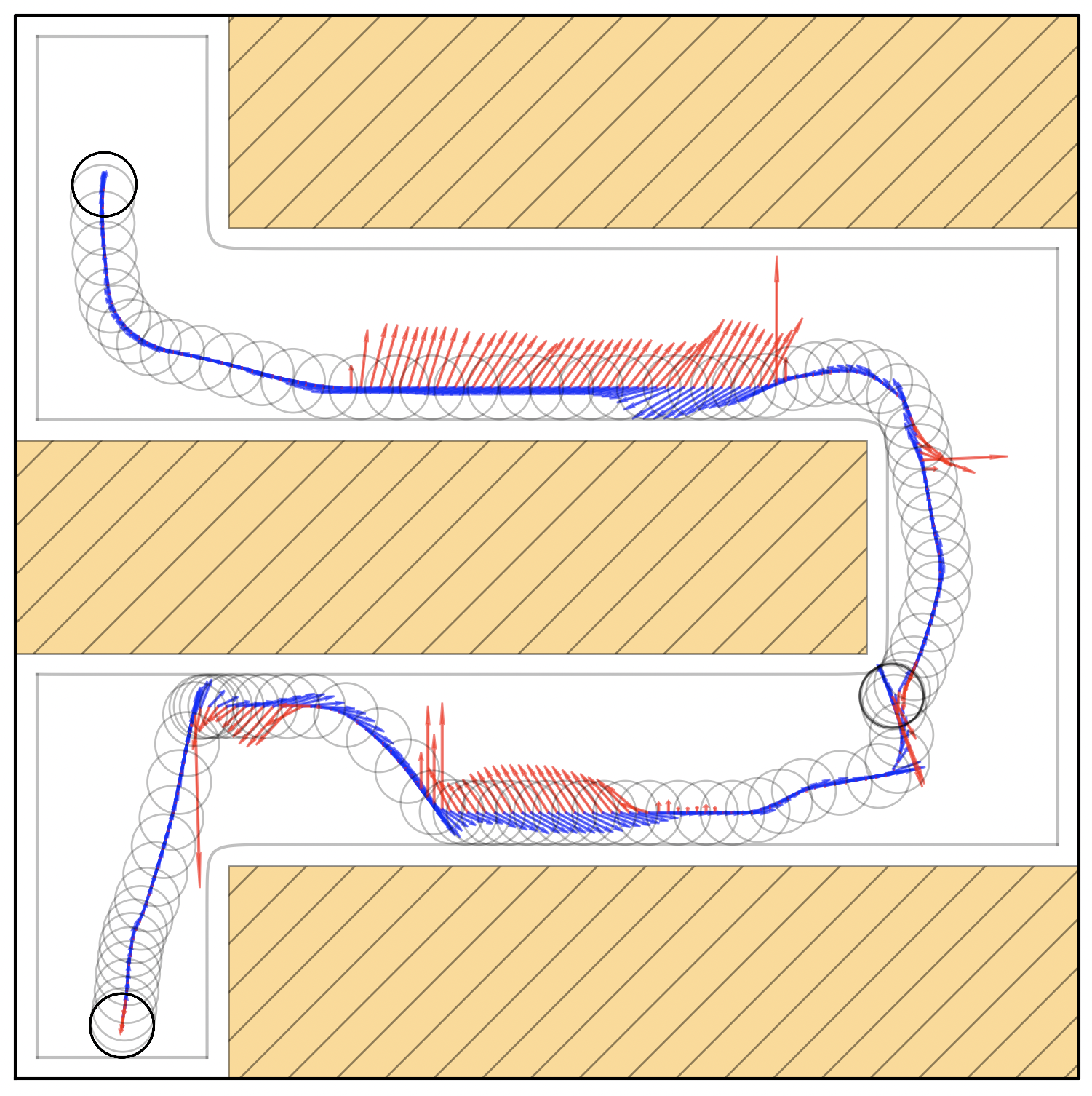}};
\draw[red, thick] (-0.4,-2) -- (0.8,-2)--(0.8, -1.3)node[above] {$M$}--(-0.4, -1.3)--(-0.4,-2);
\draw[red,thick,dashed](2.3,1.5) circle (0.7)node[right, xshift=0.6cm] {$O$};
\end{tikzpicture}
}
\caption{\small Results of the force feedback $F$ for SCF and JCF with respect to the same control input (i.e. same desired velocity from the user). Red boxes show the differences between the two designs.}
\label{Fig:2D_simulation}
\end{figure*}

%% file: hardware_rate.tex
\begin{figure*}[ht]
\centering
\subfloat[\scriptsize]
{
\begin{tikzpicture}
\node[inner sep=0pt] (pica){
\includegraphics[trim={0cm 0cm 0cm 0cm,clip},width=0.23\textwidth]{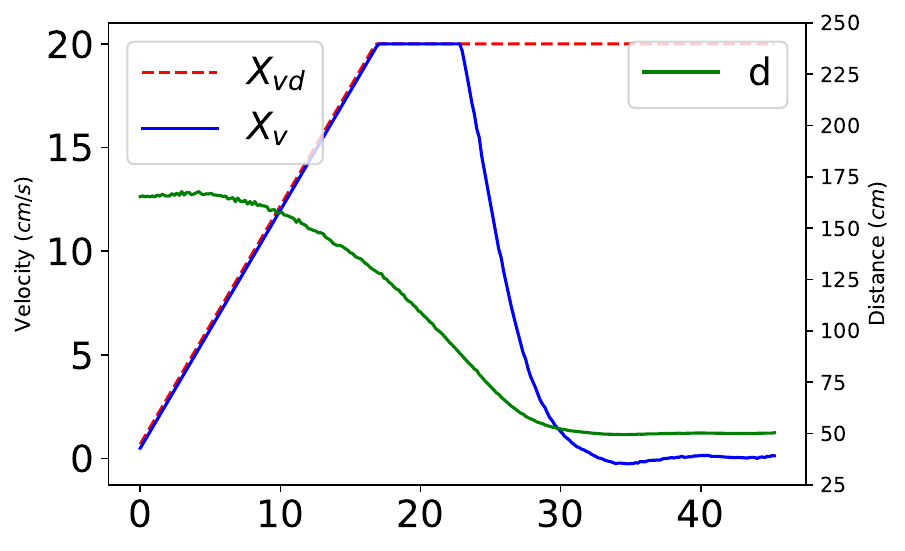}};
\end{tikzpicture}
}
\subfloat[\scriptsize]{
\begin{tikzpicture}
\node[inner sep=0pt] (b){\label{Fig:hardware_rate_b}
\includegraphics[width=0.23\textwidth,trim={0cm 0cm 0cm 0cm}]{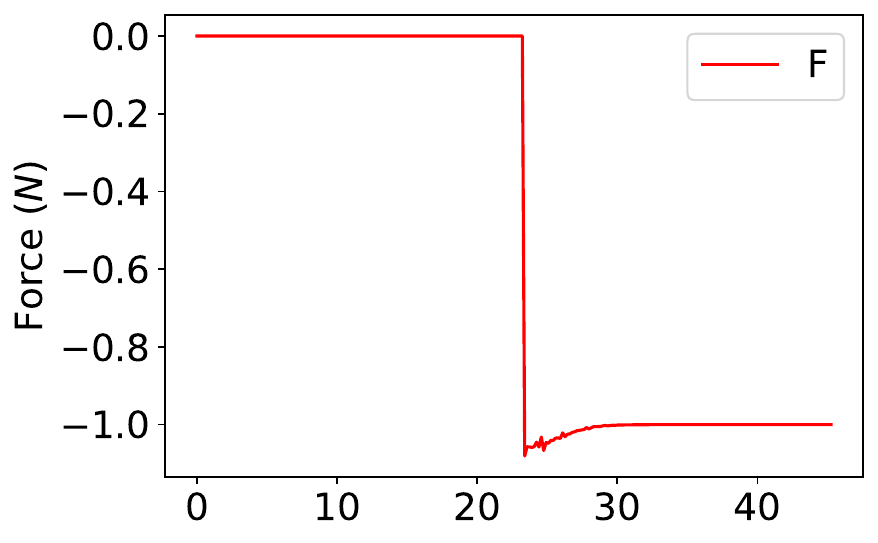}};
\draw[red, thick] (0.2,-0.85) -- (0.7,-0.85)--(0.7, -0.3)node[above] {$P$}--(0.2, -0.3)--(0.2,-0.85);
\end{tikzpicture}

}
\subfloat[\scriptsize]
{
\begin{tikzpicture}
\node[inner sep=0pt] (pica){
\includegraphics[trim={0cm 0cm 0cm 0cm,clip},width=0.23\textwidth]{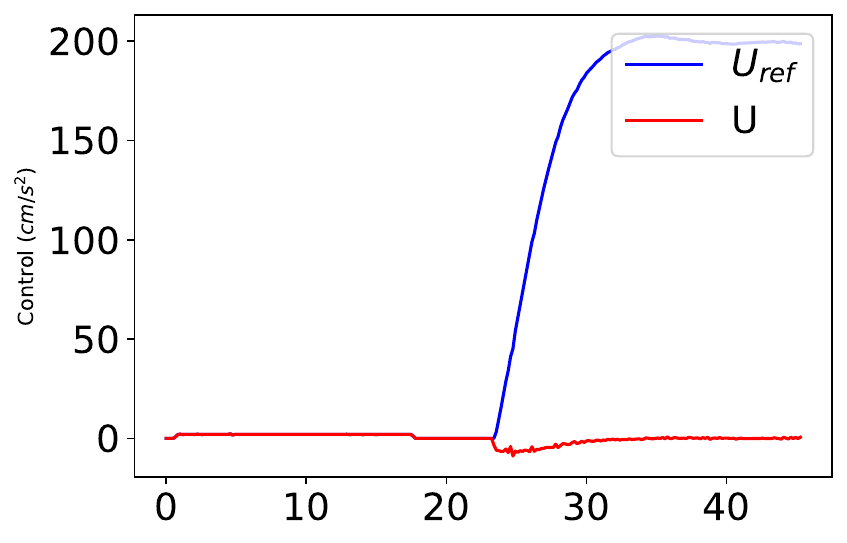}};
\end{tikzpicture}
}
\subfloat[\scriptsize]{
\includegraphics[width=0.23\textwidth,trim={0cm 0cm 0cm 0cm}]{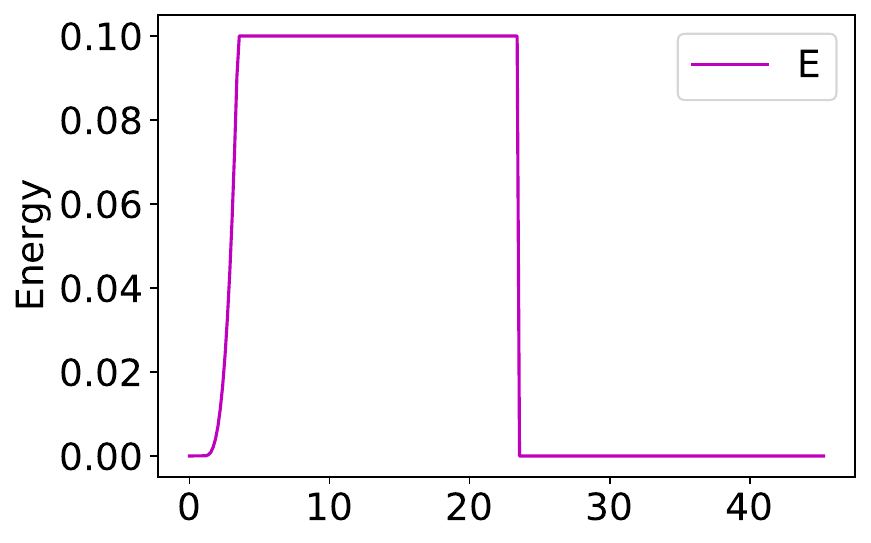}
}
\caption{\small Results of the hardware implementation from the predefined rate control signal.}
\label{Fig:hardware_rate}
\vspace{-1cm}
\end{figure*}

%% file: hardware_user.tex
\begin{figure*}[ht]
\centering
\subfloat[\scriptsize]
{
\begin{tikzpicture}
\node[inner sep=0pt] (pica){
\includegraphics[trim={0cm 0cm 0cm 0cm,clip},width=0.23\textwidth]{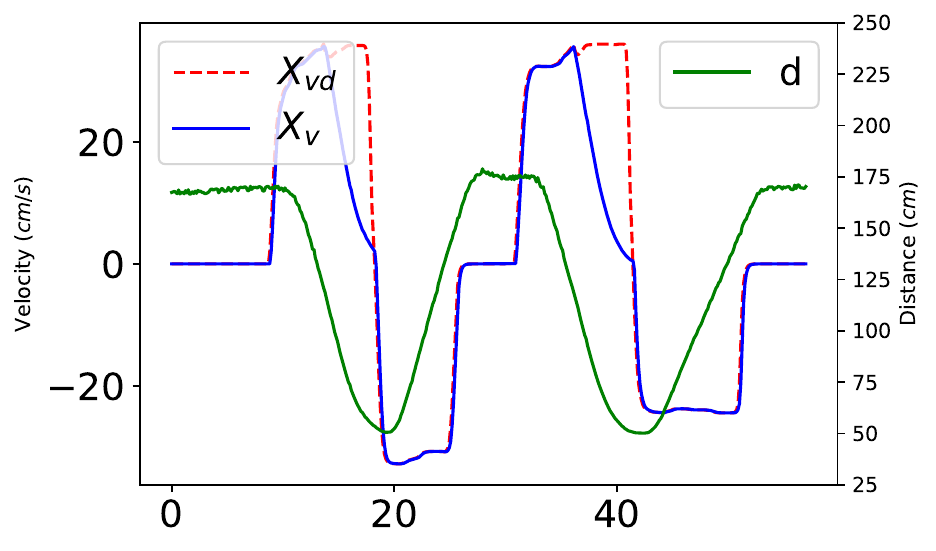}};
\draw[red, thick] (0.4,0) -- (0.7,0)node[right] {$Q$}--(0.7, 0.95)--(0.4, 0.95)--(0.4, 0);
\end{tikzpicture}
}
\subfloat[\scriptsize]{
\includegraphics[width=0.23\textwidth,trim={0cm 0cm 0cm 0cm}]{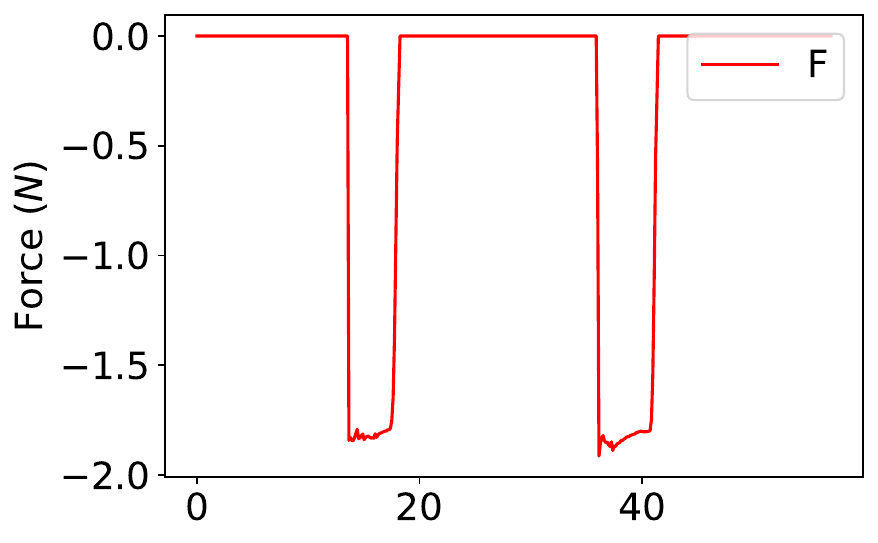}
}
\subfloat[\scriptsize]
{
\begin{tikzpicture}
\node[inner sep=0pt] (pica){
\includegraphics[trim={0cm 0cm 0cm 0cm,clip},width=0.23\textwidth]{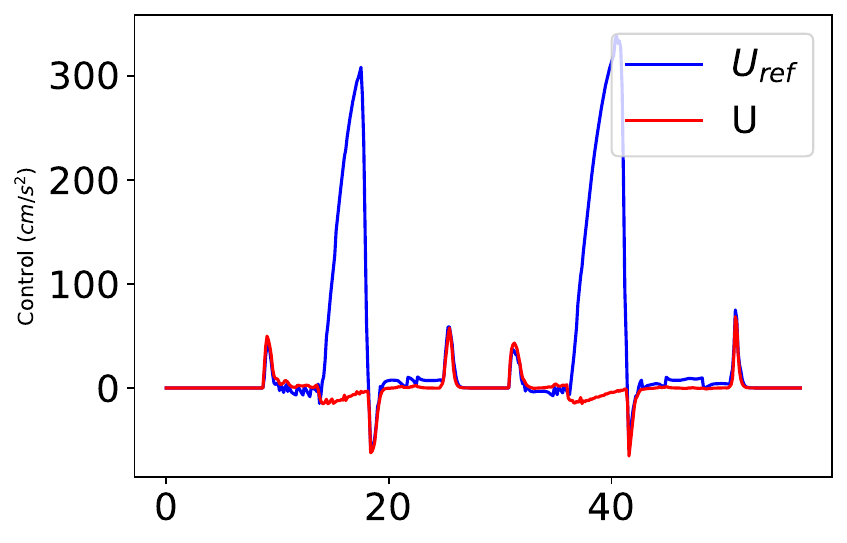}};
\end{tikzpicture}
}
\subfloat[\scriptsize]{
\includegraphics[width=0.23\textwidth,trim={0cm 0cm 0cm 0cm}]{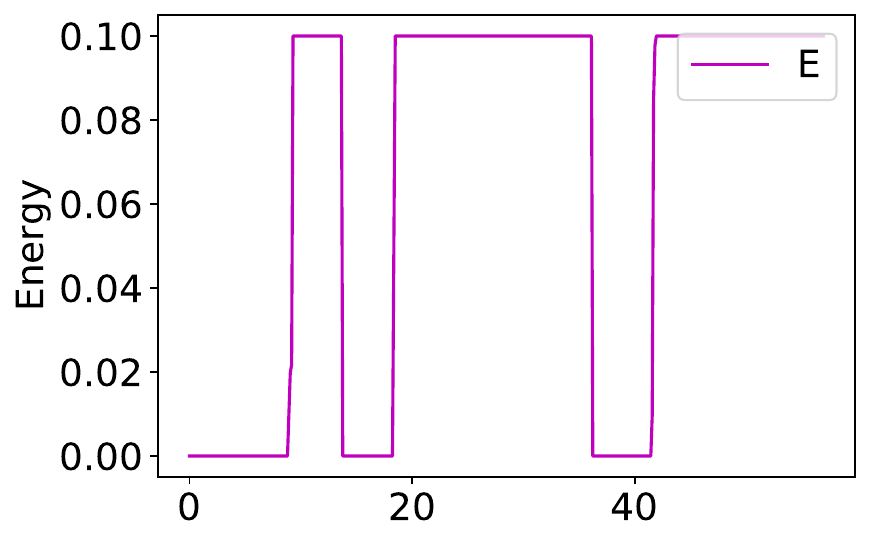}
}
\caption{\small Results of the hardware implementation from the user's control signal using a haptic interface.}
\label{Fig:hardware_user}
\end{figure*}

%% file: appendix.tex
\section {The quadratic constraint is not active}\label{appendix_1}
In the case that the quadratic constraint is not active, the Lagrangian of the constrained problem is 
\begin{multline}
    L(u,\lambda)=\norm{u -u_{\textrm{ref}}}^{2}\\
    +\lambda(x\transpose_v\partial^2_{x_p} h x_v + (\partial_{x_p} h)\transpose u+ k_1h + k_2(\partial_{x_p} h)\transpose x_v)
\end{multline}

We need to make $\nabla_{u,\lambda} L =0$. Taking the derivatives with respect to $u$ and $\lambda$, we get:
\begin{equation}
    \nabla_u L = 2(u-u_{\textrm{ref}}) + \partial_{x_p} h \lambda
\end{equation}
and
\begin{equation}
    \nabla_\lambda L = x\transpose_v\partial^2_{x_p} h x_v+ (\partial_{x_p} h)\transpose u+ k_1h + k_2(\partial_{x_p} h)\transpose x_v
\end{equation}
By calculating $\nabla_u L =0$ and $\nabla_\lambda L =0$, we can get $\lambda$:
\begin{multline}
     \lambda = \\ \frac{2(\partial_{x_p} h)\transpose u_{\textrm{ref}}+x\transpose_v\partial^2_{x_p} h x_v + k_1h + k_2(\partial_{x_p} h)\transpose x_v}{(\partial_{x_p} h)\transpose \partial_{x_p} h}
\end{multline}
   
then we can get $u$:
\begin{multline}
    u = (I-\frac{\partial_{x_p} h(\partial_{x_p} h)\transpose}{\norm{\partial_{x_p} h}^2})u_{\textrm{ref}}\\ - \frac{(x\transpose_v\partial^2_{x_p} h x_v + k_1h + k_2(\partial_{x_p} h)\transpose x_v)}{\norm{\partial_{x_p} h}^2}\partial_{x_p}h 
\end{multline}

\section{Only the quadratic constraint is active}\label{appendix_2}
The Lagrangian of the constrained problem is 
\begin{multline}
    L(F,u,\lambda) = \norm{u -u_{\textrm{ref}}}^2 + \norm{F-(u-u_{\textrm{ref}})}^2\\ +
    \lambda (\norm{F}^2 - \frac{1}{k^2}( \frac{2kE}{\Delta_t} +x_{vd}\transpose x_{vd} -2kk_vx_v\transpose u))\\
    = 2\norm{u-u_{\textrm{ref}}}^2 + (1+\lambda)\norm{F}^2 - 2F\transpose(u -u_{\textrm{ref}})\\-
    \lambda\frac{1}{k^2}( \frac{2kE}{\Delta_t} +x_{vd}\transpose x_{vd}) +
\lambda \frac{2k_v}{k}x_v\transpose u
\end{multline}

Take the derivatives with respect to $u$, $F$ and $\lambda$, we get:
\begin{equation}
    \nabla_u L = 4(u -u_{\textrm{ref}})-2F+\frac{2k_v}{k}x_v \lambda,
\end{equation}
\begin{equation}
    \nabla_F L = 2(1+\lambda)F - 2(u-u_{\textrm{ref}}),
\end{equation}
\begin{equation}
    \nabla_\lambda L = \norm{F}^2 - \frac{1}{k^2}( \frac{2kE}{\Delta_t} +x_{vd}\transpose x_{vd}) + \frac{2k_v}{k}x_v\transpose u,
\end{equation}
then we can solve $u$, $F$, and $\lambda$.

To simplify the problem, we make $C_1 =\frac{1}{k^2}(\frac{2kE}{\Delta_t}+x_{vd}\transpose x_{vd})$ and $C_2 = \frac{2k_v}{k}x_v $, then we have:
\begin{equation}
    4(u -u_{\textrm{ref}})-2F+C_2 \lambda=0
\end{equation}
\begin{equation}
     (1+\lambda)F - (u-u_{\textrm{ref}}) = 0
\end{equation}   
\begin{equation}\label{34}
    F\transpose F -C_1 +C_2u = 0
\end{equation}
the we can get:
\begin{equation}
   F = \frac{-C_2\lambda}{2+4\lambda}
\end{equation}
\begin{equation}
   u = u_{\textrm{ref}} - \frac{C_2\lambda(1+\lambda)}{2+4\lambda}
\end{equation}
we substitute $F$ and $u$ to \eqref{34}, then we get:
\begin{equation}
    (\frac{C_2\lambda}{2+4\lambda})^2 -C_1 + C_2(u_{\textrm{ref}} - \frac{C_2\lambda(1+\lambda)}{2+4\lambda})=0
\end{equation}
\begin{multline}
     4C_2^2\lambda^3 + (5C_2^2 + 16C_1 -16C_2u_{\textrm{ref}})\lambda^2 \\
    (2C_2^2 + 16C_1 -16C_2u_{\textrm{ref}})\lambda \\ + 4(C_1 -C_2u_{\textrm{ref}})=0
\end{multline}
This results in a third order polynomial equation, we can solve $\lambda$ and
then obtain $F$ and $u$. We should note that $2+4\lambda= 0$ is a special case of our method using the Lagrangian Multiplier. When $x_2=0$, we have $u = u_{\textrm{ref}}$, and $F=0$. When $x_v \neq 0$, it has no solutions for $F$ and $u$.

\section {The quadratic constraint and one linear constraint are active}\label{appendix_3}
In this case, we can also use the Lagrangian multiplier to solve this problem. $u'_{\textrm{ref}}$, $u''_{\textrm{ref}}$ and $u_\parallel$ are all known parameters. Then we have:

\begin{multline}
    L(u', F', F'', \lambda) = \norm{ u' - u'_{\textrm{ref}}}^2 + \norm{F'-( u' -
u'_{\textrm{ref}})}^2\\ + \norm{F''-( u_\parallel-u''_{\textrm{ref}}) }^2+ 
    \lambda(\norm{F'}^2+\norm{F''}^2 - \\ \frac{1}{k^2}( \frac{2kE}{\Delta_t} +x_{vd}\transpose x_{vd} -2kk_vx_v\transpose (\partial_{x_p}h u_\parallel +U_\bot u')))
\end{multline}

Take the derivatives with respect to $u'$, $F'$, $F''$ and $\lambda$, we get:
\begin{equation}
    \nabla_{u'} L = 4(u' -u'_{\textrm{ref}})-2F'+\frac{2k_v}{k}x_v\transpose U_\bot \lambda,
\end{equation}
\begin{equation}
    \nabla_{F'} L = 2(1+\lambda)F' - 2(u'-u'_{\textrm{ref}}),
\end{equation}
\begin{equation}
    \nabla_{F''} L = 2(1+\lambda)F'' - 2(u_\parallel-u''_{\textrm{ref}}),
\end{equation}
\begin{multline}
    \nabla_\lambda L = \norm{F'}^2+\norm{F''}^2 - \frac{1}{k^2}( \frac{2kE}{\Delta_t} +x_{2d}\transpose x_{vd} \\-2kk_vx_v\transpose (\partial_{x_p}h u_\parallel +U_\bot u')),
\end{multline}
then we make:
\begin{equation}
    4(u' -u'_{\textrm{ref}})-2F'+\frac{2k_v}{k}x_v\transpose U_\bot\lambda =0
\end{equation}
\begin{equation}
    2(1+\lambda)F' - 2(u'-u'_{\textrm{ref}}) =0
\end{equation}
\begin{equation}
 2(1+\lambda)F'' - 2(u_\parallel-u''_{ref})=0
\end{equation}
\begin{multline}\label{43}
    \norm{F'}^2+\norm{F''}^2 - \frac{1}{k^2}( \frac{2kE}{\Delta_t} +x_{vd}\transpose x_{vd} \\-2kk_vx_v\transpose (\partial_{x_p}h u_\parallel +U_\bot u'))=0
\end{multline}
then we can solve $u'$, $F'$, $F''$ and $\lambda$. We make $C_1 =\frac{1}{k^2}(\frac{2kE}{\Delta_t}+x_{vd}\transpose x_{vd})$ and $C_2 = \frac{2k_v}{k}x_2\transpose U_\bot $, then we have:
\begin{equation}
F' = \frac{-C_2\lambda}{2+4\lambda}
\end{equation}
\begin{equation}
    F'' = \frac{u_\parallel-u''_{\textrm{ref}}}{1+\lambda}
\end{equation}
\begin{equation}
    u'= u'_{\textrm{ref}} - \frac{C_2\lambda(1+\lambda)}{2+4\lambda}
\end{equation}
then we substitute $u'$, $F'$ and $F''$ to \eqref{43}, we have:
\begin{multline}
    (\frac{-C_2\lambda}{2+4\lambda})^2 + (\frac{u_\parallel-u''_{\textrm{ref}}}{1+\lambda})^2
    -C_1 + \frac{2k_v}{k}x_v\transpose \partial_{x_p}h u_\parallel +\\ C_2(u'_{\textrm{ref}} - \frac{C_2\lambda(1+\lambda)}{2+4\lambda}) =0
\end{multline}
this will form a 5th-order polynomial equation, we can solve $\lambda$. Here we make $C_3= (u_\parallel-u''_{\textrm{ref}})^2$, $C_4 = \frac{2k_v}{k}x_v\transpose \partial_{x_p}h u_\parallel -C_1 + C_2u'_{\textrm{ref}}$, then we have:
\begin{multline}
    4C_2^2\lambda^5 + (13C_2^2 -16C_4)\lambda^4 + (16C_2^2 -48C_4)\lambda^3\\ +(9C_2^2 -16C_3 -52C_4)\lambda^2\\ + (2C_2^2 -16C_3 -24C_4)\lambda -4C_3-4C_4 = 0
\end{multline}

We note that $2+4\lambda= 0$ and $1+\lambda=0$ are special cases of our method using the Lagrangian Multiplier. When $2+4\lambda= 0$ and $x_v=0$, we have $F'=0$, $F''= 2(u_\parallel-u''_{\textrm{ref}})$ and $u'= u'_{\textrm{ref}}$. When $2+4\lambda= 0$ and $x_v \neq 0$, there are no solutions; hence we can deduce that $x_v\neq 0\implies \lambda\neq 0.5$. When $1+ \lambda = 0$, it has solutions only if $u = u_{\textrm{ref}}$, $F' = -\frac{k_v}{k}x_v\transpose
U_\bot $.